\documentclass[10pt,journal,compsoc]{IEEEtran}
\usepackage{amsmath,amsfonts,amssymb,bm}
\usepackage{color}
\usepackage{graphicx}
\usepackage{hyperref}
\usepackage{enumitem}
\usepackage{multirow}
\usepackage{subfigure}
\usepackage{svg}
\usepackage{tabularx}
\usepackage{url}
\usepackage{verbatim}

\def\ie{{\it i.e.}}
\def\eg{{\it e.g.}}
\def\et{{\it et al.}}

\def\A{{\mathbf{A}}}
\def\dA{\Delta\mathbf{A}}

\def\tA{\widetilde{\mathbf{A}}}
\def\D{{\mathbf{D}}}
\def\tD{\widetilde{\mathbf{D}}}
\def\E{{\mathbf{E}}}

\def\H{{\mathbf{H}}}
\def\dH{\Delta\mathbf{H}}
\def\tH{\widetilde{\mathbf{H}}}
\def\h{{\mathbf{h}}}
\def\dh{\delta\mathbf{h}}

\def\I{{\mathbf{I}}}
\def\S{{\mathbf{S}}}
\def\W{{\mathbf{W}}}
\def\x{{\mathbf{x}}}
\def\X{{\mathbf{X}}}
\def\tX{\widetilde{\mathbf{X}}}
\def\Y{{\mathbf{Y}}}

\newcolumntype{Y}{>{\centering\arraybackslash}X}

\newtheorem{definition}{Definition}
\newtheorem{theorem}{Theorem}

\newenvironment{proof}{{\noindent\it Proof}\quad}{\hfill $\square$\par}

\ifCLASSOPTIONcompsoc
  \usepackage[nocompress]{cite}
\else
  \usepackage{cite}
\fi

\begin{document}

\title{Self-Supervised Graph Representation Learning via Topology Transformations}

\author{Xiang~Gao,~\IEEEmembership{Student Member,~IEEE,}
        Wei~Hu,~\IEEEmembership{Senior~Member,~IEEE,}
        and~Guo-Jun~Qi,~\IEEEmembership{Fellow,~IEEE}
\IEEEcompsocitemizethanks{
\IEEEcompsocthanksitem X. Gao and W. Hu are with Wangxuan Institute of Computer Technology, Peking University, No. 128 Zhongguancun North Street, Beijing, China.\protect\\
E-mail: \{gyshgx868, forhuwei\}@pku.edu.cn
\IEEEcompsocthanksitem G.-J. Qi is with the Futurewei Seattle Cloud Lab, Seattle, WA.\protect\\
E-mail: guojunq@gmail.com
\IEEEcompsocthanksitem This work is supported by the National Key R\&D project of China under contract No. 2021YFF0901502, and National Natural Science Foundation of China under contract No. 61972009.
\IEEEcompsocthanksitem The corresponding author is W. Hu (forhuwei@pku.edu.cn).
The code is available at: \url{https://github.com/gyshgx868/topo-ter}
}
}

\markboth{IEEE TRANSACTIONS ON KNOWLEDGE AND DATA ENGINEERING}%
{Shell \MakeLowercase{\textit{et al.}}: Bare Demo of IEEEtran.cls for Computer Society Journals}

\IEEEtitleabstractindextext{%
\begin{abstract}
We present the Topology Transformation Equivariant Representation learning, a general paradigm of self-supervised learning for node representations of graph data to enable the wide applicability of Graph Convolutional Neural Networks (GCNNs). 
We formalize the proposed model from an information-theoretic perspective, by maximizing the mutual information between topology transformations and node representations before and after the transformations.
We derive that maximizing such mutual information can be relaxed to minimizing the cross entropy between the applied topology transformation and its estimation from node representations.
In particular, we seek to sample a subset of node pairs from the original graph and flip the edge connectivity between each pair to transform the graph topology.
Then, we self-train a representation encoder to learn node representations by reconstructing the topology transformations from the feature representations of the original and transformed graphs.
In experiments, we apply the proposed model to the downstream node classification, graph classification and link prediction tasks, and results show that the proposed method outperforms the state-of-the-art unsupervised approaches.
\end{abstract}

}

\maketitle

\IEEEdisplaynontitleabstractindextext

\IEEEpeerreviewmaketitle

\IEEEraisesectionheading{\section{Introduction}\label{sec:intro}}

Graphs provide a natural and efficient representation for non-Euclidean data, such as brain networks, social networks, citation networks, and 3D point clouds. Graph Convolutional Neural Networks (GCNNs) \cite{bronstein2017geometric} have been proposed to generalize the CNNs to learn representations from non-Euclidean data, which has made significant advances in various applications such as node classification \cite{kipf2017semi,velivckovic2018graph,xu2019graph} and graph classification \cite{xu2018powerful}.
However, most existing GCNNs are trained in a supervised fashion, requiring a large amount of labeled data for network training.
This limits the applications of the GCNNs since it is often costly to collect adequately labeled data, especially on large-scale graphs.
Hence, self-supervised learning is required to learn graph feature representations by exploring the dependencies of unlabeled data in an unsupervised fashion, 
which enables the discovery of intrinsic graph structures and thus adapts to various downstream tasks.

Various attempts have been made to explore self-supervisory signals for representation learning.
The self-supervised learning framework requires only unlabeled data in order to design a \textit{pretext} learning task, where the target objective is optimized without any supervision \cite{kolesnikov2019revisiting}.
Self-supervised learning models can be categorized into three classes \cite{liu2020self}: generative, adversarial, and contrastive.
Generative models are often based on Auto-regressive models \cite{van2016pixel,oord16conditional,you2018graphrnn}, flow-based models \cite{dinh2014nice,dinh2016density}, and Auto-Encoding (AE) models \cite{ballard1987modular,kingma13auto,kipf2016variational} to generate or reconstruct data from latent representations.
Adversarial models extract feature representations in an unsupervised fashion by generating data from input noises via a pair of generator and discriminator \cite{goodfellow2014generative,wang2018graphgan}.
Contrastive models aim to train an encoder to be \textit{contrastive} between the representations of positive samples and negative samples \cite{hjelm2018learning,velickovic2019deep,peng2020graph,sun2020infograph}.

Recently, many approaches have sought to learn \textit{transformation equivariant representations (TERs)} to further improve the quality of unsupervised representation learning.
It assumes that the learned representations equivarying to transformations are able to encode the intrinsic structures of data such that the transformations can be reconstructed from the representations before and after transformations \cite{qi2019learning}.
Learning TERs traces back to Hinton's seminal work on learning transformation capsules \cite{hinton2011transforming}, and embodies a variety of methods developed for Euclidean data \cite{kivinen2011transformation,sohn2012learning,schmidt2012learning,skibbe2013spherical,gens2014deep,lenc2015understanding,dieleman2015rotation,dieleman2016exploiting,zhang2019aet,qi2019avt,wang2020transformation}.
Further, Gao \textit{et al.} \cite{gao2020graphter} extend transformation equivariant representation learning to non-Euclidean domain, which formalizes Graph Transformation Equivariant Representation (GraphTER) learning by auto-encoding node-wise transformations in an unsupervised fashion.
Nevertheless, only transformations on node features are explored, while the underlying graph may vary implicitly. The graph topology has not been fully explored yet, which however is crucial in graph representation learning. 

To this end, we propose a self-supervised Topology Transformation Equivariant Representation learning to infer expressive graph feature representations by estimating topology transformations.
As the topology of graph data is critical in graph-based machine learning, this motivates us to study the properties of graph topology including especially the topology transformation equivariance.
In particular, the topology of some graph data such as citation networks and social networks may transform over time, leading to the change of node features\footnote{For instance, with the development of research in different fields, more interdisciplinary citations will appear in citation networks; the personal relationships in social networks may change over time.}.
When the topology transforms, the node features transform equivariantly in the representation space.
Hence, instead of transforming node features as in the GraphTER, the proposed method studies the transformation equivariant representation learning by transforming the graph topology, \ie, adding or removing edges to perturb the graph structure.
Then the same input signals are attached to the resultant graph topologies, resulting in different graph representations.
This provides an insight into how the same input signals associated with different graph topologies would lead to equivariant representations, enhancing the fusion of node feature and graph topology in GCNNs. 

Formally, we formulate the proposed model from an information-theoretic perspective, aiming to maximize the mutual information between topology transformations and feature representations with respect to the original and transformed graphs. 
We derive that maximizing such mutual information can be relaxed to the cross entropy minimization between the applied topology transformations and the estimation from the learned representations of graph data under the topological transformations.


Specifically, given an input graph and its associated node features, we first sample a subset of node pairs from the graph and flip the edge connectivity between each pair at a perturbation rate, leading to a transformed graph with attached node features.
Then, we design a graph-convolutional auto-encoder architecture, where the encoder learns the node-wise representations over the original and transformed graphs respectively, and the decoder predicts the topology transformations of edge connectivity from both representations by minimizing the cross entropy between the applied and estimated transformations.
Experimental results demonstrate that the proposed method outperforms the state-of-the-art unsupervised models, and even achieves comparable results to the (semi-)supervised approaches in node classification, graph classification and link prediction tasks at times.

The proposed method distinguishes from our previous work GraphTER \cite{gao2020graphter} mainly in two aspects.
1) We formulate our model from an information-theoretic perspective by maximizing the mutual information between representations and transformations, which provides a theoretical derivation for the training objective and generalizes transformations to more general forms.
In contrast, GraphTER directly minimizes the MSE between the estimated and ground-truth transformations, which lacks theoretical explanation and is limited to parametric transformations;
2) We explicitly exploit transformations in the {\it graph topology}, which is crucial in graph representation learning and explores how the same input signals associated with different graph topologies would lead to equivariant representations, thus enabling deeper fusion of node features and the graph topology in GCNNs.
In contrast, GraphTER focuses on learning equivariant representations of nodes under node-wise transformations.

Our main contributions are summarized as follows.
\begin{itemize}
    \item We propose a self-supervised paradigm of the Topology Transformation Equivariant Representation learning to infer expressive node feature representations, which characterizes the intrinsic structures of graphs and the associated features by exploring the graph transformations of connectivity topology. 
    
    \item We formulate the Topology Transformation Equivariant Representation learning from an information-theoretic perspective, by maximizing the mutual information between feature representations and topology transformations, which is proved to relax to the cross entropy minimization between the applied transformations and the prediction in an end-to-end graph-convolutional auto-encoder architecture. 
  
    \item Experiments demonstrate that the proposed method outperforms the state-of-the-art unsupervised methods in node classification, graph classification, and link prediction.
\end{itemize}

The remainder of this paper is organized as follows.
We first review related works in Sec.~\ref{sec:related}. 
Then we formalize our model in Sec.~\ref{sec:formulation} and present the algorithm in Sec.~\ref{sec:method}.
Finally, experimental results and conclusions are presented in Sec.~\ref{sec:experiments} and Sec.~\ref{sec:conclusion}, respectively.

\section{Related Work}
\label{sec:related}

We review previous works on relevant unsupervised/self-supervised feature representation learning, including graph auto-encoders, graph generative models, graph contrastive learning, as well as transformation equivariant representation learning.

\subsection{Graph Auto-Encoders}

Graph Auto-Encoders (GAEs) are the most representative unsupervised methods.
GAEs encode graph data into feature space via an encoder and reconstruct the input graph data from the encoded feature representations via a decoder.
Kipf \textit{et al.} \cite{kipf2016variational} first integrate the GCN \cite{kipf2017semi} into an auto-encoder framework to learn graph representations in an unsupervised manner by reconstructing the adjacency matrix.
Variational GAE (VGAE) \cite{kipf2016variational} is a variational version of GAE to learn the distribution of data.
Cao \textit{et al.} \cite{cao2016deep} proposed to employ the stacked denoising auto-encoder \cite{vincent2008extracting} to reconstruct the positive pointwise mutual information (PPMI) matrix to capture the correlation of node pairs.
Wang \textit{et al.} \cite{wang2016structural} employ the stacked auto-encoders to preserve the first-order proximity and the second-order proximity of nodes jointly.
Qu \textit{et al.} \cite{qu2019gmnn} proposed the Graph Markov Neural Network (GMNN) to model the joint distribution of object labels with a conditional random field for semi-supervised and unsupervised graph representation learning.
Jiang \textit{et al.} \cite{jiang2020co} introduce a graph auto-encoder framework that can be used for unsupervised link prediction, where the encoder embeds both nodes and edges to a latent feature space simultaneously.
Compared with these auto-encoder-based methods that aim to decode the original graph, our proposed model makes better use of the topology information of graphs and the transformation equivariance property. 

\subsection{Graph Generative Networks}

Graph Generative Networks aim to learn the generative distribution of graphs by encoding graphs into hidden representations and generate graph structures given hidden representations.
The graph generative networks can be classified into two categories \cite{wu2020comprehensive}: sequential approaches and global approaches.
Sequential approaches generate nodes and edges step by step.
Deep Generative Model of Graphs (DeepGMG) \cite{li2018learning} assumes that the probability of a graph is the sum over all possible node permutations, and generates graphs by making a sequence of decisions.
You \textit{et al.} \cite{you2018graphrnn} proposed GraphRNN model to generate nodes from a graph-level RNN and edges from an edge-level RNN.
Global approaches generate an entire graph at once.
Molecular GAN (MolGAN) \cite{de2018molgan} combines Relational Graph Convolutional Networks (R-GCNs) \cite{schlichtkrull2018modeling}, GANs \cite{gulrajani2017improved}, and reinforcement learning objectives to generate graphs with desired properties.
NetGAN \cite{bojchevski2018netgan} combines LSTMs \cite{hochreiter1997long} with the Wasserstein GANs \cite{arjovsky2017wasserstein} to generate graphs from a random-walk-based approach.

\subsection{Graph Contrastive Learning}

An important paradigm called contrastive learning aims to train an encoder to be \textit{contrastive} between the representations of positive samples and negative samples \cite{hjelm2018learning,bachman2019learning,chen2020simple,he2020momentum,hu2021adco}.
Recent contrastive learning frameworks for graph data can be divided into two categories \cite{liu2020self}: \textit{context-instance} contrast and \textit{context-context} contrast.
Context-instance contrast focuses on modeling the relationships between the local feature of a sample and its global context representation.
Deep InfoMax (DIM) \cite{hjelm2018learning} first maximizes the mutual information between a local patch and its global context through a contrastive learning task.
Deep Graph InfoMax (DGI) \cite{velickovic2019deep} extends DIM to graph-structured data to learn node-level feature representations. 
Sun \textit{et al.} \cite{sun2020infograph} proposed an InfoGraph model to maximize the mutual information between the representations of entire graphs and the representations of substructures of different granularity.
Peng {\it et al.} \cite{peng2020graph} proposed a Graphical Mutual Information (GMI) approach to maximize the mutual information of both features and edges between inputs and outputs.
Compared with context-instance methods, context-context contrast studies the relationships between the global representations of different samples.
Caron \textit{et al.} \cite{caron2018deep} proposed a Deep Cluster approach to cluster encoded representations and produces pseudo labels for each sample, and then predicts whether two samples are from the same cluster.
Sun \textit{et al.} \cite{sun2020multi} adopts a self-supervised pre-training paradigm as in DeepCluster \cite{caron2018deep} for better semi-supervised prediction in GCNNs.
Qiu \textit{et al.} \cite{qiu2020gcc} designs the pre-training task as subgraph instance discrimination in and across networks to empower graph neural networks to learn intrinsic structural representations.
You \textit{et al.} \cite{you2020graph} proposed a novel graph contrastive learning framework for GNN pre-training to facilitate invariant representation learning, where four types of graph augmentation methods are designed to incorporate various priors.

\subsection{Transformation Equivariant Representations}

Many approaches have sought to learn transformation equivariant representations, which has been advocated in Hinton's seminal work on learning transformation capsules \cite{hinton2011transforming}.
Following this, a variety of approaches have been proposed to learn transformation equivariant representations \cite{gens2014deep,dieleman2015rotation,dieleman2016exploiting,cohen2016group,lenssen2018group}.
To generalize to generic transformations, Zhang \textit{et al.} \cite{zhang2019aet} proposed to learn unsupervised feature representations via Auto-Encoding Transformations (AET) by estimating transformations from the learned feature representations of both the original and transformed images, while Qi \textit{et al.} \cite{qi2019avt} extend AET from an information-theoretic perspective by maximizing the lower bound of mutual information between transformations and representations.
Wang \textit{et al.} \cite{wang2020transformation} extend the AET to Generative Adversarial Networks (GANs) for unsupervised image synthesis and representation learning.
Gao \textit{et al.} \cite{gao2020graphter} introduce the GraphTER model that extends AET to graph-structured data, which is formalized by auto-encoding node-wise transformations in an unsupervised manner.
De \textit{et al.} \cite{de2020gauge} proposed Gauge Equivariant Mesh CNNs which generalize GCNNs to apply anisotropic gauge equivariant kernels.
Fuchs \textit{et al.} \cite{fuchs2020se} introduce a self-attention mechanism specifically for 3D point cloud data, which adheres to equivariance constraints, improving robustness to nuisance transformations.
Haan \textit{et al.} \cite{de2020natural} proposed a Natural Graph Network (NGN) that can be used to describe maximally flexible global and local equivariance.
Satorras \textit{et al.} \cite{satorras2021n} present an E(n) equivariant graph neural network that is translation, rotation and reflection equivariant.
Gao \textit{et al.} \cite{gao2021self} proposed to learn multi-view representations by decoding the 3D transformations of 3D objects from multiple 2D views.
Our approach belongs to this family of methods.
The key difference is that, we propose to learn topology transformation equivariant representations by transforming the graph topology.

\section{The Proposed Formulation}
\label{sec:formulation}

In this section, we first introduce the preliminaries in Sec.~\ref{subsec:preliminary}, and define the topology transformation in Sec.~\ref{subsec:topo_transform}.
Then we formulate the proposed method in Sec.~\ref{subsec:formulation}.
Further, some analysis of the proposed model are presented in Sec.~\ref{subsec:analysis}.

\subsection{Preliminary}
\label{subsec:preliminary}

We consider an undirected graph $\mathcal{G}=\{\mathcal{V},\mathcal{E},\mathbf{A}\}$ composed of a node set $\mathcal{V}$ of cardinality $|\mathcal{V}|=N$, an edge set $\mathcal{E}$ connecting nodes of cardinality $|\mathcal{E}|=M$.
$\mathbf{A}$ is a real symmetric $N \times N$ matrix that encodes the graph structure, where $a_{i,j}=1$ if there exists an edge $(i,j)$ between nodes $i$ and $j$, and $a_{i,j}=0$ otherwise.
\textit{Graph signal} refers to data that reside on the nodes of a graph $\mathcal{G}$, denoted by $\mathbf{X} \in \mathbb{R}^{N \times C}$ with the $i$-th row representing the $C$-dimensional graph signal on the $i$-th node of $\mathcal{V}$.

\subsection{Topology Transformation}
\label{subsec:topo_transform}

\begin{figure*}[t]
  \centering
  \includegraphics[width=0.8\textwidth]{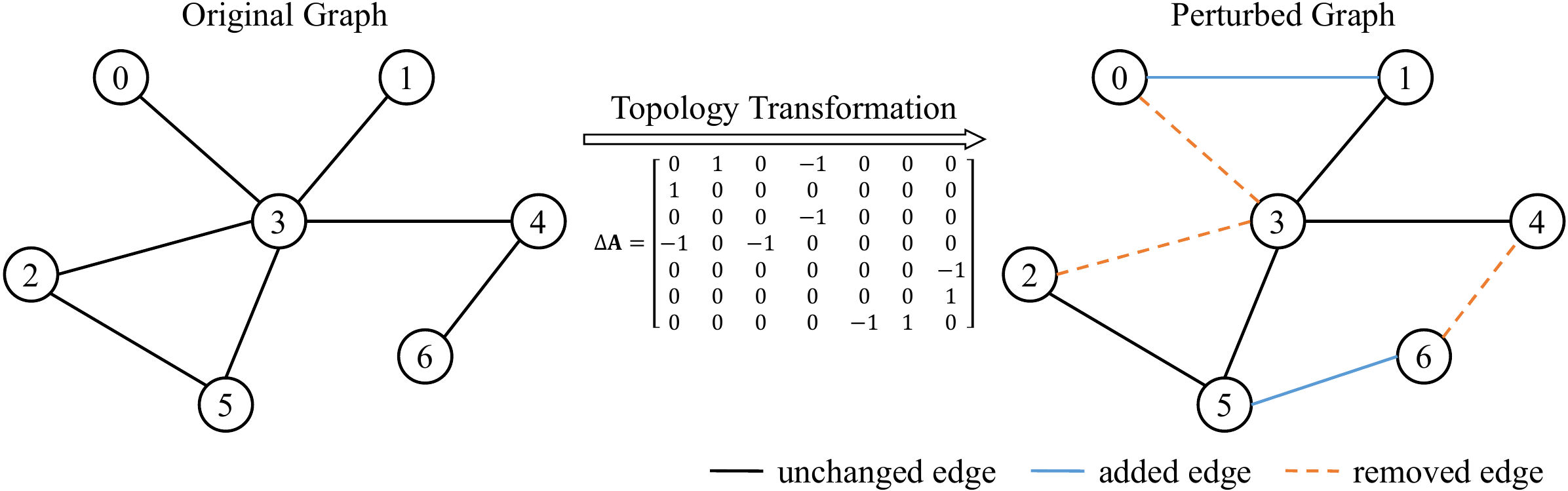}
  \caption{An example of graphs before and after topology transformations.}
  \label{fig:topo_perturb}
\end{figure*}

We define the topology transformation $\mathbf{t}$ as adding or removing edges from the original edge set $\mathcal{E}$ in graph $\mathcal{G}$.
This can be done by sampling, i.i.d., a \textit{switch} parameter $\sigma_{i,j}$ as in \cite{velickovic2019deep}, which determines whether to modify edge $(i,j)$ in the adjacency matrix.
Assuming a Bernoulli distribution $\mathcal{B}(p)$, where $p$ denotes the probability of each edge being modified, we draw a random matrix $\Sigma=\left\{\sigma_{i,j}\right\}_{N \times N}$ from $\mathcal{B}(p)$, \ie, $\Sigma \sim \mathcal{B}(p)$.
We then acquire the perturbed adjacency matrix as
\begin{equation}
  \tA = \A \oplus \Sigma,
\end{equation}
where $\oplus$ is the exclusive OR (XOR) operation.
This strategy produces a transformed graph through the topology transformation $\mathbf{t}$, \ie, $\tA = \mathbf{t}(\A)$.
Here, the edge perturbation probability of $p=0$ corresponds to a non-transformed adjacency matrix, which is a special case of an identity transformation to $\A$.

The transformed adjacency matrix $\tA$ can also be written as the sum of the original adjacency matrix $\A $ and a topology perturbation matrix $\dA$:
\begin{equation}
  \tA  = \A + \dA,
\end{equation}
where $\dA= \{\delta a_{i,j}\}_{N \times N}$ encodes the perturbation of edges, with $\delta a_{i,j} \in \{-1, 0, 1\}$.
As shown in Fig.~\ref{fig:topo_perturb}, when $\delta a_{i,j}=0$, the edge between node $i$ and node $j$ keeps unchanged (\ie, black solid lines); when $\delta a_{i,j} = -1$ or $1$, it means removing (\ie, orange dotted lines) or adding (\ie, blue solid lines) the edge between node $i$ and node $j$, respectively.

\subsection{The Formulation}
\label{subsec:formulation}

\begin{definition}
Given a pair of graph signal and adjacency matrix $(\X,\A)$, and a pair of graph signal and \textit{transformed} adjacency matrix $(\X,\tA)$ by a topology transformation $\mathbf{t}(\cdot)$, a function $E(\cdot)$ is \textit{transformation equivariant} if it satisfies
\begin{equation}
  E(\X,\tA)=E\left(\X,\mathbf{t}(\A)\right)=\rho(\mathbf{t})\left[E(\X,\A)\right],
  \label{eq:ter}
\end{equation}
where $\rho(\mathbf{t})[\cdot]$ is a homomorphism of transformation $\mathbf{t}$ in the representation space.
\end{definition}


Let us denote $ \H=E(\X,\A), \; \text{and} \; \tH=E(\X,\tA)$. 
We seek to learn an encoder $E: (\X,\A) \mapsto \H;(\X,\tA) \mapsto \tH$ that maps both the original and transformed sample to representations $\{\H,\tH\}$ equivariant to the sampled transformation $\mathbf{t}$, whose information can thus be inferred from the representations via a decoder $D: (\tH,\H) \mapsto \widehat{\dA}$ as much as possible. 
From an information-theoretic perspective, this requires $(\H,\dA)$ should jointly contain all necessary information about $\tH$.

Then a natural choice to formalize the topology transformation equivariance is the {\it mutual information} $I(\H,\dA;\tH)$ between $(\H,\dA)$ and $\tH$. The larger the mutual information is, the more knowledge about $\dA$ can be inferred from the representations $\{\H,\tH\}$. 
Hence, we propose to maximize the mutual information to learn the topology transformation equivariant representations as follows:
\begin{equation}
  \max_{\theta} \; I(\H,\dA;\tH),
  \label{eq:objective}
\end{equation}
where $\theta$ denotes the parameters of the auto-encoder network.




Nevertheless, it is difficult to compute the mutual information directly.
Instead, we derive that maximizing the mutual information can be relaxed to minimizing the cross entropy, as described in the following theorem. 

\vspace{0.1in}
\begin{theorem}
The maximization of the mutual information $I(\H,\dA;\tH)$ can be  relaxed to the minimization of the cross entropy $H(p \parallel q)$ between the probability distributions $p(\dA,\tH,\H)$ and $q_{\theta}(\widehat{\dA}|\tH,\H)$:
\begin{equation}
  \min_{\theta} \; H\left(p(\dA,\tH,\H) \; \| \; q_{\theta}(\widehat{\dA}|\tH,\H)\right).
  \label{eq:cross_entropy}
\end{equation}
\end{theorem}

\begin{proof}
By using the chain rule of mutual information, we have
\begin{equation}\nonumber
    I(\H,\dA;\tH) = I(\dA;\tH|\H)+I(\H;\tH) \ge I(\dA;\tH|\H).
\end{equation}
Thus the mutual information $I(\dA;\tH|\H)$ is the lower bound of the mutual information $I(\H,\dA;\tH)$ that attains its minimum value when $I(\H;\tH)=0$.

Therefore, we relax the objective to maximizing the lower bound mutual information $I(\dA;\tH|\H)$ between the transformed representation $\tH$ and the topology transformation $\dA$:
\begin{equation}\nonumber
  I(\dA;\tH|\H) = H(\dA|\H) - H(\dA|\tH,\H),
\end{equation}
where $H(\cdot)$ denotes the conditional entropy. Since $\dA$ and $\H$ are independent, we have $H(\dA|\H)=H(\dA)$.
Hence, maximizing $I(\dA;\tH|\H)$ becomes
\begin{equation}
  \min_{\theta} \; H(\dA|\tH,\H).
  \label{eq:min_conditional_entropy}
\end{equation}
However, this minimization problem requires to evaluate the posterior probability distribution $q(\dA|\tH,\H)$, which is often difficult to calculate directly.
We next introduce a conditional probability distribution $q_{\theta}(\widehat{\dA}|\tH,\H)$ to approximate the intractable posterior $q(\dA|\tH,\H)$ with an estimated transformation $\widehat{\dA}$, \textit{i.e.},
\begin{equation}\nonumber
\begin{split}
  &H(\dA|\tH,\H) = -\underset{p(\dA,\tH,\H)}{\mathbb{E}}\log q(\dA|\tH,\H) \\
  =&-\underset{p(\dA,\tH,\H)}{\mathbb{E}}\log q_{\theta}(\widehat{\dA}|\tH,\H) \\
  & -\underset{p(\tH,\H)}{\mathbb{E}}D_{\text{KL}}\left(q(\dA|\tH,\H)\parallel q_{\theta}(\widehat{\dA}|\tH,\H)\right) \\
  \leq& -\underset{p(\dA,\tH,\H)}{\mathbb{E}}\log q_{\theta}(\widehat{\dA}|\tH,\H),
\end{split}
\end{equation}
where $D_{\text{KL}}\left(q(\dA|\tH,\H)\parallel q_{\theta}(\widehat{\dA}|\tH,\H)\right)$ denotes the Kullback-Leibler divergence of $q(\dA|\tH,\H)$ and $q_{\theta}(\widehat{\dA}|\tH,\H)$ that is non-negative.
Thus, the minimization problem in Eq.~(\ref{eq:min_conditional_entropy}) is converted to minimizing the cross entropy as the upper bound:
\begin{equation}\nonumber
  \begin{split}
    \min_{\theta} \; & H\left(p(\dA,\tH,\H) \; \| \; q_{\theta}(\widehat{\dA}|\tH,\H)\right) \\
    \triangleq & -\underset{p(\dA,\tH,\H)}{\mathbb{E}} \log q_{\theta}(\widehat{\dA}|\tH,\H).
  \end{split}
\end{equation}
Hence, we relax the maximization problem in Eq.~(\ref{eq:objective}) to the optimization in Eq.~(\ref{eq:cross_entropy}).
\end{proof}

Based on \textbf{Theorem 1}, we train the decoder $D$ to learn the distribution $q_{\theta}(\widehat{\dA}|\tH,\H)$ so as to estimate the topology transformation $\widehat{\dA}$ from the encoded $\{\tH,\H\}$, where the input pairs of original and transformed graph representations $\{\tH,\H\}$ as well as the ground truth target $\dA$ can be sampled tractably from the factorization of $p(\dA,\tH,\H)\triangleq p(\dA)p(\H)p(\tH|\dA,\H)$. This allows us to minimize the {\it cross entropy} between $p(\dA)$ and $q_{\theta}(\widehat{\dA}|\tH,\H)$ with the training triplets $(\tH,\H;\dA)$ drawn from the tractable factorization of $p(\dA,\tH,\H)$.
Hence, we formulate the Topology Transformation Equivariant Representation learning as the joint optimization of the representation encoder $E$ and the transformation decoder $D$. 


\subsection{Analysis}
\label{subsec:analysis}

Since we learn the topology transformation equivariant representations by maximizing the mutual information between representations and transformations as discussed in Sec.~\ref{subsec:formulation}, the proposed model {\it approximately} learns equivariant representations for graph data as defined in Eq.~(\ref{eq:ter}) via the optimization in Eq.~(\ref{eq:cross_entropy}).
This distinguishes from relevant attempts on equivariance for graph data \cite{fuchs2020se,de2020gauge,de2020natural,satorras2021n}, which design {\it exact} equivariant kernels for graph data.

Though equivariance is explicitly satisfied in \cite{fuchs2020se,de2020gauge,de2020natural,satorras2021n}, the equivariant kernels depend on the modalities of the input data, and thus are not generalizable to various tasks. 
This is because they exploit representative information in different data modalities instead of their commonalities, \eg, 3D coordinates of point clouds \cite{fuchs2020se} and the angles of two neighboring points in meshes \cite{de2020gauge}.
For instance, Fuchs \et \cite{fuchs2020se} proposed an SE(3)-equivariant convolutional network by constructing an attention-based SE(3)-Transformer specifically for 3D point clouds.
Haan \et \cite{de2020gauge} presented an anisotropic gauge equivariant kernel for mesh data, which is applied into graph convolutional networks and results in equivalent outputs regardless of the arbitrary choice of kernel orientation.
Also, Haan \et \cite{de2020natural} proposed Natural Graph Networks for isomorphic graphs that are equivariant to node permutations.
Further, Satorras \et \cite{satorras2021n} came up with a new graph convolution kernel, which makes inputs equivariant to parameterized orthogonal transformations (\eg, rotations, translations and reflections) and permutations for data such as molecules.
In all, these methods design equivariant network kernels tailored for specific graph data, which have been applied in supervised graph representation learning.

In contrast, we propose an {\it unsupervised} model that generalizes to different downstream tasks of various graph data, without restrictions to transformations or types of graph data, while providing good approximations of equivariant representations via the proposed effective optimization.



\section{The Proposed Algorithm}
\label{sec:method}

\begin{figure*}[t]
  \centering
  \includegraphics[width=0.8\textwidth]{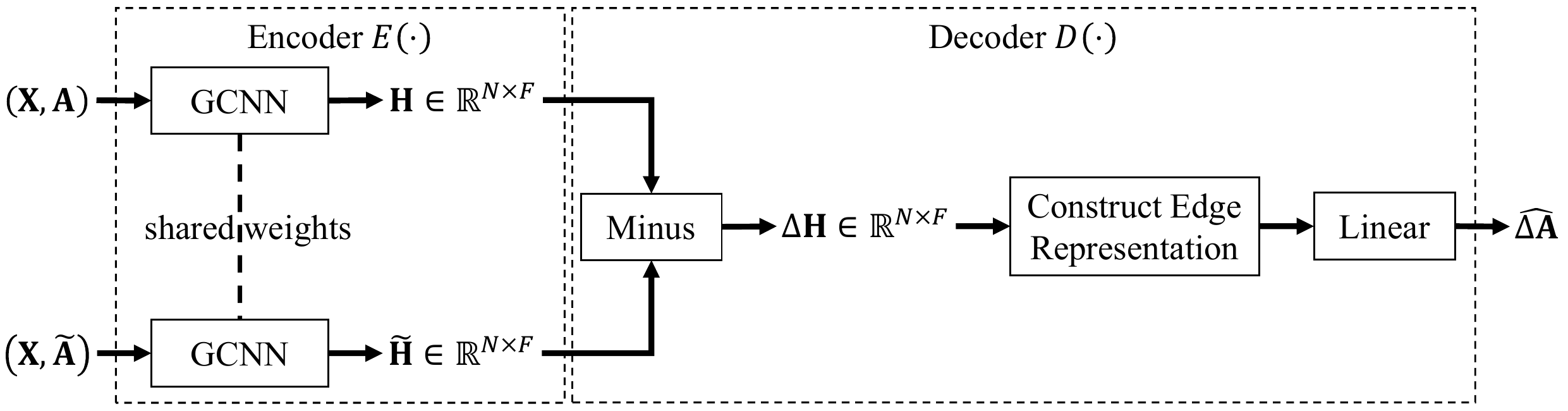}
  \caption{The architecture of the proposed model.}
  \label{fig:framework}
\end{figure*}

We design a graph-convolutional auto-encoder network for the proposed model, as illustrated in Fig.~\ref{fig:framework}.
Given a graph signal $\X$ associated with a graph $\mathcal{G}=\{\mathcal{V},\mathcal{E},\mathbf{A}\}$, the proposed unsupervised learning algorithm consists of three steps:
1) topology transformation, which samples and perturbs some edges from all node pairs to acquire a transformed adjacency matrix $\tA$;
2) representation encoding, which extracts the feature representations of graph signals before and after the topology transformation;
3) transformation decoding, which estimates the topology transformation parameters from the learned feature representations. 
We elaborate on the three steps as follows. 


\subsection{Topology Transformation} 
We randomly sample a subset of node pairs from all node pairs for topology perturbation---adding or removing edges, which not only enables to characterize local graph structures at various scales, but also reduces the number of edge transformation parameters to estimate for computational efficiency.
In practice, in each iteration of training, we sample {\it all} the node pairs with connected edges $\S_1$ (\ie, $\S_1=\mathcal{E}$), and randomly sample a subset of disconnected node pairs $\S_0$, \ie,
\begin{equation}
  \S_0=\left\{(i,j) \big| a_{i,j}=0 \right\}, \S_1=\left\{(i,j) \big| a_{i,j}=1 \right\},
\end{equation}
where $|\S_0|=|\S_1|=M$. 
Next, we randomly split $\S_0$ and $\S_1$ into two disjoint sets, respectively, \ie,
\begin{equation}
  \begin{split}
    \S_i=\bigg\{\S_i^{(1)},\S_i^{(2)} \; \big | \; & \S_i^{(1)} \cap \S_i^{(2)} = \varnothing, \S_i^{(1)} \cup \S_i^{(2)} = \S_i, \\
    & |\S_i^{(1)}|=r \cdot |\S_i| \bigg\}, i \in \{0,1\},
  \end{split}
  \label{eq:split}
\end{equation}
where $r$ is the {\it edge perturbation rate}.
Then, for each node pair $(i,j)$ in $\S_0^{(1)}$ and $\S_1^{(1)}$, we {\it flip} the corresponding entry in the original graph adjacency matrix. That is, if $a_{i,j}=0$, then we set $\tilde{a}_{i,j}=1$; otherwise, we set $\tilde{a}_{i,j}=0$.
For each node pair $(i,j)$ in $\S_0^{(2)}$ and $\S_1^{(2)}$, we keep the original connectivities unchanged, \ie, $\tilde{a}_{i,j}=a_{i,j}$.

This leads to the transformed adjacency matrix $\tA$, as well as the sampled transformation parameters by accessing $\dA$ at position $(i,j)$ from $\S_0$ and $\S_1$.
Also, we can category the sampled topology transformation parameters into four types: 
\begin{enumerate}[leftmargin=12pt]
    \item add an edge to a disconnected node pair, \ie, $\{\mathbf{t}:a_{i,j}=0 \mapsto \tilde{a}_{i,j}=1, (i,j) \in \S_0^{(1)}\}$;
    
    \item delete the edge between a connected node pair, \ie, $\{\mathbf{t}:a_{i,j}=1 \mapsto \tilde{a}_{i,j}=0, (i,j) \in \S_1^{(1)}\}$;
    
    \item keep the disconnection between node pairs in $\S_0^{(2)}$, \ie, $\{\mathbf{t}:a_{i,j}=0 \mapsto \tilde{a}_{i,j}=0, (i,j) \in \S_0^{(2)}\}$; 
    
    \item keep the connection between node pairs in $\S_1^{(2)}$, \ie, $\{\mathbf{t}:a_{i,j}=1 \mapsto \tilde{a}_{i,j}=1, (i,j) \in \S_1^{(2)}\}$.
\end{enumerate}

Thus, we cast the problem of estimating transformation parameters in $\dA$ from $(\tH,\H)$ as the classification problem of the transformation parameter types. 
The percentage of these four types is $r:r:(1-r):(1-r)$. 

\subsection{Representation Encoder} We train an encoder $E: (\X,\A) \mapsto E(\X,\A)$ to encode the feature representations of each node in the graph.
As demonstrated in Fig.~\ref{fig:framework}, we leverage GCNNs with shared weights to extract feature representations of each node in the graph signal.
Taking the GCN \cite{kipf2017semi} as an example, the graph convolution in the GCN is defined as
\begin{equation}
  \H=E(\X,\A)=\D^{-\frac{1}{2}}(\A+\I)\D^{-\frac{1}{2}}\X\W,
  \label{eq:before_transform}
\end{equation}
where $\D$ is the degree matrix of $\A+\I$, $\W \in \mathbb{R}^{C \times F}$ is a learnable parameter matrix, and $\H = [\h_1,...,\h_N]^{\top} \in \mathbb{R}^{N \times F}$ denotes the node-wise feature matrix with $F$ output channels.

From Eq.~(\ref{eq:before_transform}), we see that node-wise representations from GCN are updated in two steps: 1) feature propagation and aggregation, and 2) linear transformation.

The first step aims to aggregate the features of each node $v_i$ and its local neighborhood, \eg, 1-hop neighborhood,
\begin{equation}
  \widehat{\h}_i=\frac{1}{d_i}\x_i+\sum_{i \sim j}\frac{a_{i,j}}{\sqrt{d_i \cdot d_j}}\x_j,
  \label{eq:feat_prop_aggr}
\end{equation}
where $i \sim j$ represents node $i$ and $j$ are connected, and $\widehat{\h}_i$ denotes the aggregated features of node $v_i$.
Eq.~(\ref{eq:feat_prop_aggr}) can also be expressed over the entire graph by matrix multiplication, \ie, $\widehat{\H}=\D^{-\frac{1}{2}}(\A+\I)\D^{-\frac{1}{2}}\X$.

The second step performs a linear transformation with a learnable parameter matrix $\W$ on the aggregated features $\widehat{\H}$ to generate node embeddings for the GCN layer, \ie, $\H=\widehat{\H}\W$.

Similarly, the node feature of the transformed counterpart is as follows with the shared weights $\W$.  
\begin{equation}
\begin{split}
  \tH & =E(\X,\tA) =\tD^{-\frac{1}{2}}(\tA+\I)\tD^{-\frac{1}{2}}\X\W \\
  & =\tD^{-\frac{1}{2}}(\A+\I)\tD^{-\frac{1}{2}}\X\W + \tD^{-\frac{1}{2}}\dA\tD^{-\frac{1}{2}}\X\W.
\end{split}
\label{eq:after_transform}
\end{equation}
We thus acquire the feature representations $\H$ and $\tH$ of graph signals before and after topology transformations.

\subsection{Transformation Decoder} 
Comparing Eq.~(\ref{eq:before_transform}) and Eq.~(\ref{eq:after_transform}), the prominent difference between $\tH$ and $\H$ lies in the second term of Eq.~(\ref{eq:after_transform}) featuring $\dA$.  
This enables us to train a decoder $D: (\tH,\H) \mapsto \widehat{\dA}$ to estimate the topology transformation from the joint representations before and after transformation.
We first take the difference between the extracted feature representations before and after transformations along the feature channel,
\begin{equation}
  \dH = \tH - \H = [\dh_1, ..., \dh_N]^{\top} \in \mathbb{R}^{N \times F}.
  \label{eq:feature_diff}
\end{equation}
Thus, we can predict the topology transformation between node $i$ and node $j$ through the node-wise feature difference $\dH$ by constructing the \textit{edge representation} as
\begin{equation}
  \begin{split}
    \mathbf{e}_{i,j} = \frac{\exp\{-(\dh_i - \dh_j) \odot (\dh_i - \dh_j)\}}{\| \exp\{-(\dh_i - \dh_j) \odot (\dh_i - \dh_j)\} \|_1} & , \\
    \forall (i,j) \in \S_0 \cup \S_1 & ,
  \end{split}
\end{equation}
where $\odot$ denotes the Hadamard product of two vectors to capture the feature representation, and $\|\cdot\|_1$ is the $\ell_1$-norm of a vector for normalization.
The edge representation ${\mathbf{e}}_{i,j}$ of node $i$ and $j$ is then fed into several linear layers for the prediction of the topology transformation,
\begin{equation}
  \widehat{\mathbf{y}}_{i,j}=\mathrm{softmax}\left(\mathrm{linear}({\mathbf{e}}_{i,j})\right), \quad \forall (i,j) \in \S_0 \cup \S_1,
\end{equation}
where $\mathrm{softmax}(\cdot)$ is an activation function.

According to Eq.~(\ref{eq:cross_entropy}), the entire auto-encoder network is trained by minimizing the cross entropy
\begin{equation}
  \mathcal{L}=-\underset{(i,j) \in \S_0 \cup \S_1}{\mathbb{E}}\sum_{f=1}^{4}\mathbf{y}_{i,j}^{(f)} \log \widehat{\mathbf{y}}_{i,j}^{(f)},
  \label{eq:loss}
\end{equation}
where $f$ denotes the transformation type ($f \in \{1,2,3,4\}$), and $\mathbf{y}$ is the ground-truth binary indicator ($0$ or $1$) for each transformation parameter type.

\begin{table*}[t]
\centering
\caption{Node classification accuracies (with standard deviation) in percentage on three datasets. 
$\X,\A,\Y$ denote the input data, adjacency matrix and labels respectively.
}
\label{tab:results_nc}
\begin{tabularx}{0.7\textwidth}{lYYYY}
\hline
\multicolumn{1}{c|}{\textbf{Method}} & \multicolumn{1}{c}{\textbf{Training Data}} & \multicolumn{1}{c}{\textbf{Cora}} & \multicolumn{1}{c}{\textbf{Citeseer}} & \multicolumn{1}{c}{\textbf{Pubmed}} \\ \hline
\multicolumn{5}{c}{\textbf{Semi-Supervised Methods}} \\ \hline
\multicolumn{1}{l|}{GCN \cite{kipf2017semi}} & $\X,\A,\Y$ & $81.5$ & $70.3$ & $79.0$ \\
\multicolumn{1}{l|}{MoNet \cite{monti2017geometric}} & $\X,\A,\Y$ & $81.7 \pm 0.5$ & - & $78.8 \pm 0.3$ \\
\multicolumn{1}{l|}{GAT \cite{velivckovic2018graph}} & $\X,\A,\Y$ & $83.0 \pm 0.7$ & $72.5 \pm 0.7$ & $79.0 \pm 0.3$ \\
\multicolumn{1}{l|}{SGC \cite{wu2019simplifying}} & $\X,\A,\Y$ & $81.0 \pm 0.0$ & $71.9 \pm 0.1$ & $78.9 \pm 0.0$ \\
\multicolumn{1}{l|}{GWNN \cite{xu2019graph}} & $\X,\A,\Y$ & $82.8$ & $71.7$ & $79.1$ \\
\multicolumn{1}{l|}{MixHop \cite{abu2019mixhop}} & $\X,\A,\Y$ & $81.9 \pm 0.4$ & $71.4 \pm 0.8$ & $80.8 \pm 0.6$ \\
\multicolumn{1}{l|}{DFNet \cite{wijesinghe2019dfnets}} & $\X,\A,\Y$ & $85.2 \pm 0.5$ & $74.2 \pm 0.3$ & $84.3 \pm 0.4$ \\ \hline
\multicolumn{5}{c}{\textbf{Unsupervised Methods}} \\ \hline
\multicolumn{1}{l|}{Raw Features \cite{velickovic2019deep}} & $\X$ & $47.9 \pm 0.4$ & $49.3 \pm 0.2$ & $69.1 \pm 0.3$ \\
\multicolumn{1}{l|}{DeepWalk \cite{perozzi2014deepwalk}} & $\A$ & $67.2$ & $43.2$ & $65.3$ \\
\multicolumn{1}{l|}{DeepWalk + Features \cite{velickovic2019deep}} & $\X,\A$ & $70.7 \pm 0.6$ & $51.4 \pm 0.5$ & $74.3 \pm 0.9$ \\
\multicolumn{1}{l|}{GAE \cite{kipf2016variational}} & $\X,\A$ & $80.9 \pm 0.4$ & $66.7 \pm 0.4$ & $77.1 \pm 0.7$ \\
\multicolumn{1}{l|}{VGAE \cite{kipf2016variational}} & $\X,\A$ & $80.0 \pm 0.2$ & $64.1 \pm 0.2$ & $76.9 \pm 0.1$ \\
\multicolumn{1}{l|}{DGI \cite{velickovic2019deep}} & $\X,\A$ & $81.1 \pm 0.1$ & $71.4 \pm 0.2$ & $77.0 \pm 0.2$ \\
\multicolumn{1}{l|}{GMI \cite{peng2020graph}} & $\X,\A$ & $82.2 \pm 0.2$ & $71.4 \pm 0.5$ & $78.5 \pm 0.1$ \\
\multicolumn{1}{l|}{\textbf{Ours}} & $\X,\A$ & $\mathbf{83.7 \pm 0.3}$ & $\mathbf{71.7 \pm 0.5}$ & $\mathbf{79.1 \pm 0.1}$ \\ \hline
\end{tabularx}
\end{table*}

\section{Experiments}
\label{sec:experiments}

In this section, we evaluate the proposed model on three representative downstream tasks: node classification, graph classification, and link prediction.

\subsection{Node Classification}

\subsubsection{Datasets}

We adopt three citation networks to evaluate our model: Cora, Citeseer, and Pubmed \cite{sen2008collective}.
The dataset statistics are reported in Tab.~\ref{tab:data_stat}.
The three datasets contain sparse bag-of-words feature vectors for each document and a list of citation links between documents.
We treat documents as nodes, and the citation links as (undirected) edges, leading to a binary and symmetric adjacency matrix $\A$ as in \cite{kipf2017semi}.
We follow the standard train/test split in \cite{kipf2017semi} to conduct the experiments, where the label rate denotes the number of labeled nodes that are used for training.

\begin{table}[htbp]
\centering
\caption{Dataset statistics of citation networks.}
\label{tab:data_stat}
\begin{tabular}{l|rrrrr}
\hline
\textbf{Dataset} & \multicolumn{1}{c}{\textbf{Nodes}} & \multicolumn{1}{c}{\textbf{Edges}} & \multicolumn{1}{c}{\textbf{Classes}} & \multicolumn{1}{c}{\textbf{Features}} & \multicolumn{1}{c}{\textbf{Label rate}} \\ \hline
Citeseer & 3,327 & 4,732 & 6 & 3,703 & 0.036 \\
Cora & 2,708 & 5,429 & 7 & 1,433 & 0.052 \\
Pubmed & 19,717 & 44,338 & 3 & 500 & 0.003 \\ \hline
\end{tabular}
\end{table}

\subsubsection{Implementation Details}

In this task, the auto-encoder network is trained via Adam optimizer, and the learning rate is set to $10^{-4}$.
We use the same early stopping strategy as DGI \cite{velickovic2019deep} on the observed training loss, with a patience of $20$ epochs.
We deploy one Simple Graph Convolution (SGC) layer \cite{wu2019simplifying} as our encoder, and the order of the adjacency matrix is set to $2$.
The LeakyReLU activation function with a negative slope of $0.1$ is employed after the SGC layer.
Similar to DGI \cite{velickovic2019deep}, we set the output channel $F=512$ for Cora and Citeseer dataset, and $256$ for Pubmed dataset due to memory limitations.
After the encoder, we use one linear layer to classify the transformation types.
We set the edge perturbation rate in Eq.~(\ref{eq:split}) as $r=\{0.7, 0.4, 0.7\}$ for Cora, Citeseer, and Pubmed, respectively.

During the training procedure of the classifier, the SGC layer in the encoder is used to extract graph feature representations with the weights frozen.
After the SGC layer, we apply one linear layer to map the features to the classification scores.

\subsubsection{Experimental Results}

We compare the proposed method with five unsupervised methods, including one node embedding method DeepWalk, two graph auto-encoders GAE and VGAE \cite{kipf2016variational}, and two contrastive learning methods DGI \cite{velickovic2019deep} and GMI \cite{peng2020graph}.
Additionally, we report the results of Raw Features and DeepWalk+Features \cite{perozzi2014deepwalk} under the same settings.
For fair comparison, the results of all other unsupervised methods are reproduced by using the same encoder architecture of the {proposed method} except DeepWalk and Raw Features.
We report the mean classification accuracy (with standard deviation) on the test nodes for all methods after $50$ runs of training.
As reported in Tab.~\ref{tab:results_nc}, the {proposed method} outperforms all other competing unsupervised methods on three datasets.
Further, the proposed unsupervised method also achieves comparable performance with semi-supervised results.
This significantly closes the gap between unsupervised approaches and the semi-supervised methods.

Moreover, we compare the proposed {method} with two contrastive learning methods DGI and GMI in terms of the model complexity, as reported in Tab.~\ref{tab:size_nc}.
The number of parameters in our model is less than that of DGI and even less than half of that of GMI, which further shows the {proposed model} is lightweight.

\begin{table}[t]
\centering
\caption{Model size comparison among DGI, GMI, and Ours.}
\label{tab:size_nc}
\begin{tabular}{c|ccc}
\hline
Model & DGI & GMI & Ours \\ \hline
No. of Parameters & $996,354$ & $1,730,052$ & $\mathbf{736,260}$ \\ \hline
\end{tabular}
\end{table}

\begin{figure*}[t]
  \centering
  \subfigure[Cora]{
  \includegraphics[width=0.3\textwidth]{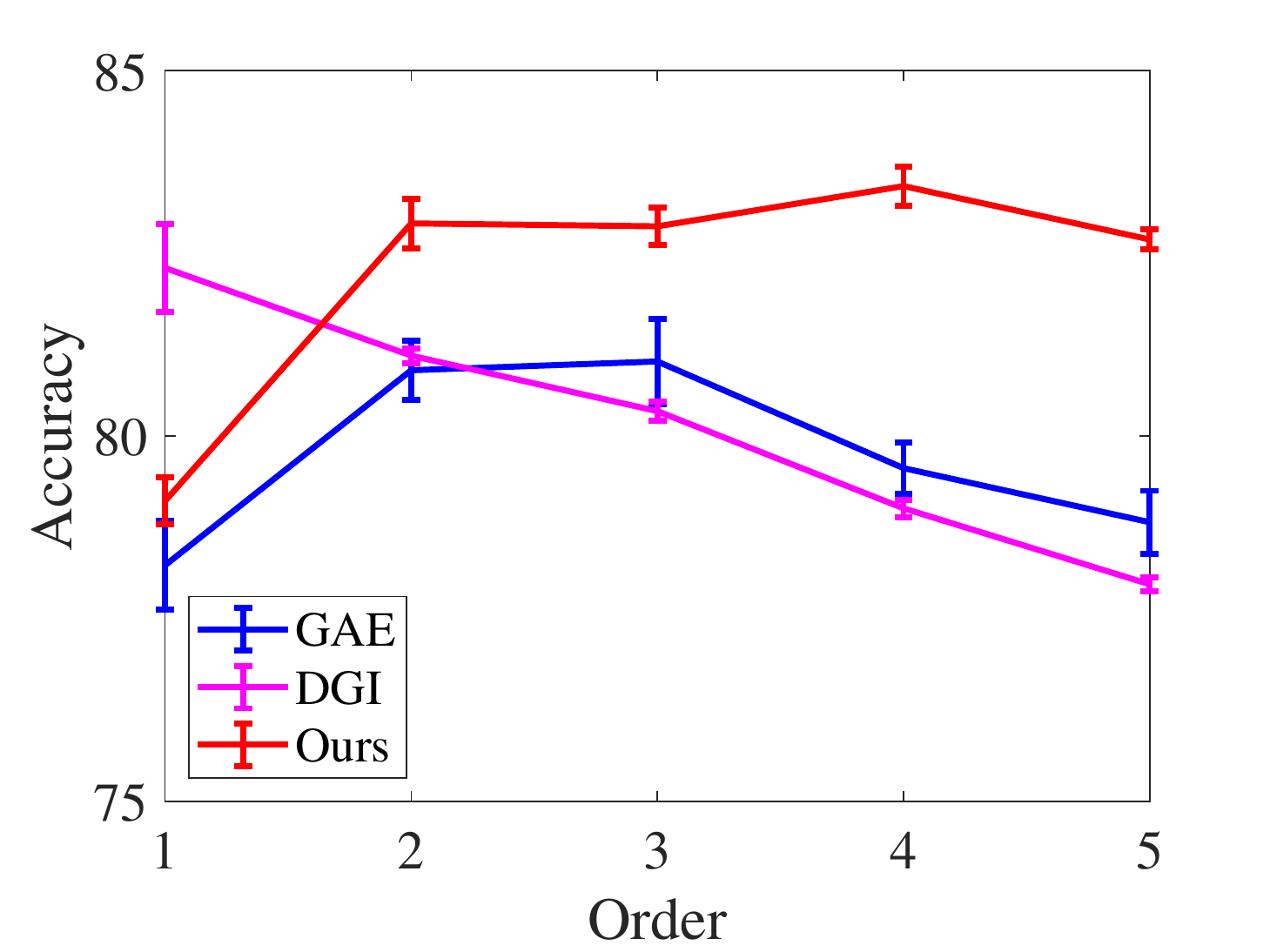}
  \label{subfig:cora_order}
  }
  \subfigure[Citeseer]{
  \includegraphics[width=0.3\textwidth]{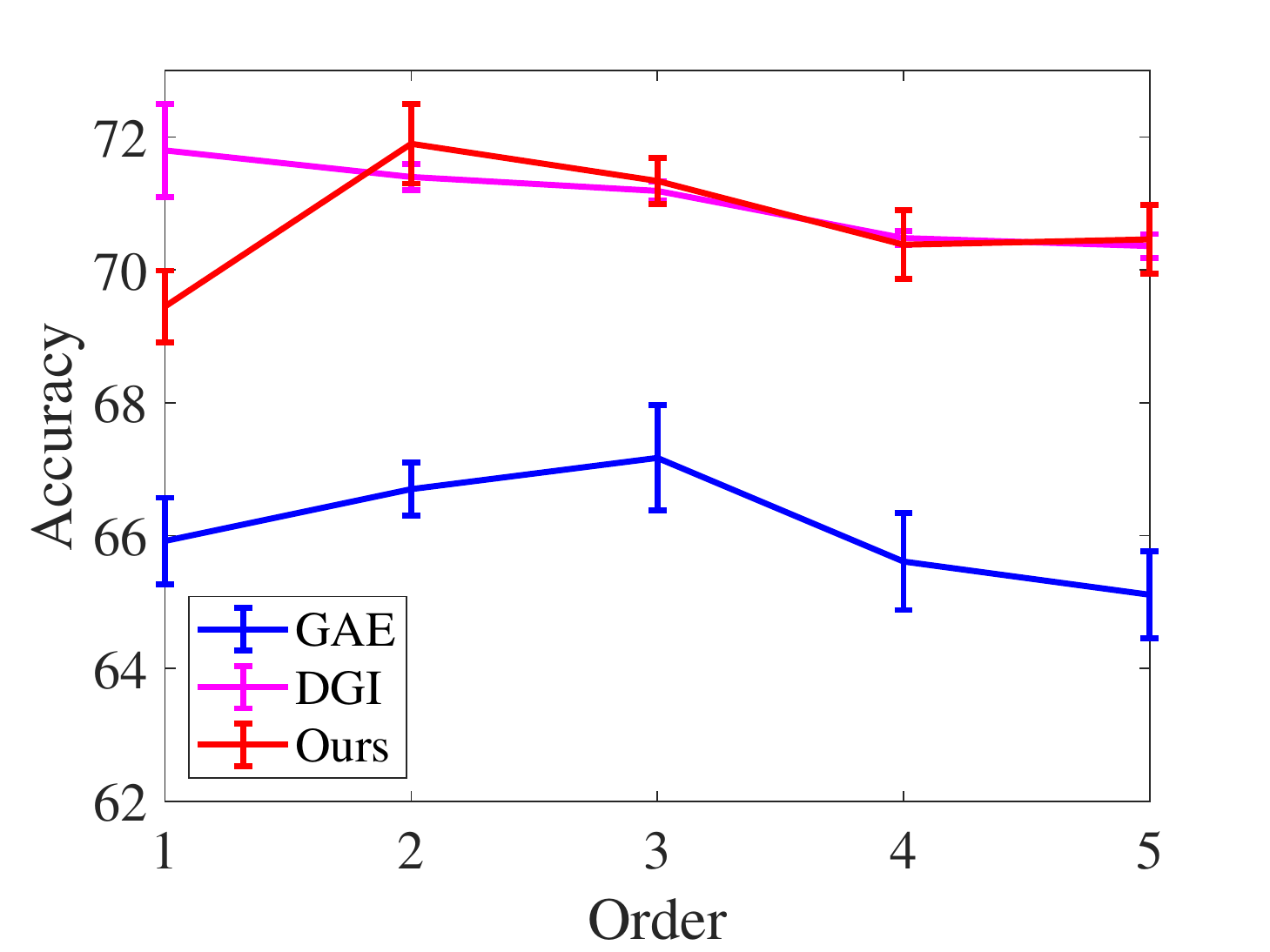}
  \label{subfig:citeseer_order}
  }
  \subfigure[Pubmed]{
  \includegraphics[width=0.3\textwidth]{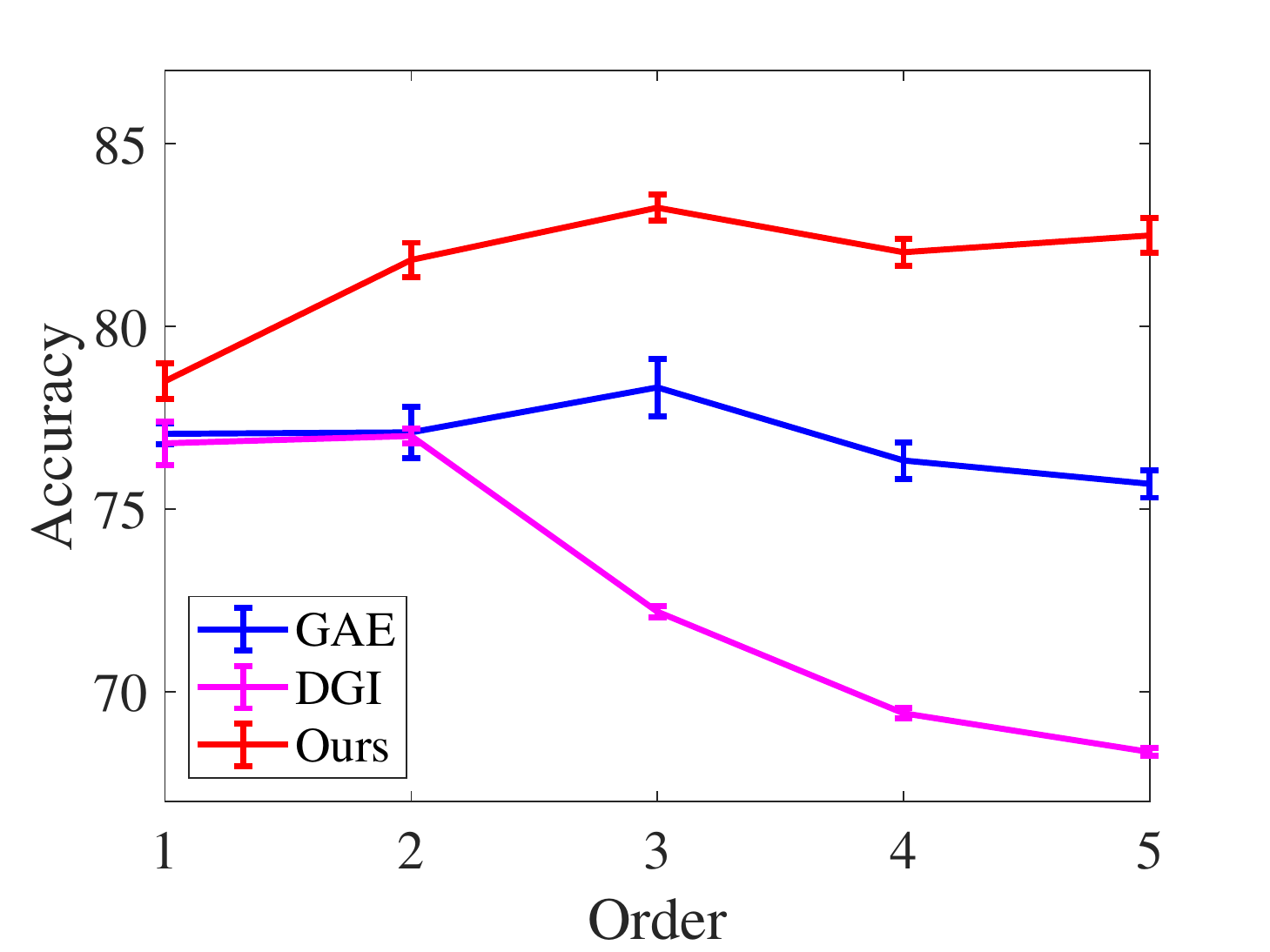}
  \label{subfig:pubmed_order}
  }
  \caption{Node classification accuracies under different orders of the adjacency matrix on the Cora, Citeseer, and Pubmed datasets.}
  \label{fig:order}
\end{figure*}

\subsubsection{Experiments on Different Orders of The Adjacency Matrix}
\label{app:order}

As presented in Sec.~\ref{subsec:topo_transform}, we perturb the $1$-hop neighborhoods via the proposed topology transformations, leading to possibly significant changes in the graph topology. 
This increases the difficulties of predicting the topology transformations when using one-layer GCN \cite{kipf2017semi} by aggregating the $1$-hop neighborhood information.
Therefore, we employ one Simple Graph Convolution (SGC) layer \cite{wu2019simplifying} with order $k$ as our encoder $E(\cdot)$, where the output feature representations aggregate multi-hop neighborhood information.
Formally, the SGC layer is defined as
\begin{equation}
  \H=E(\X,\A)=\left(\D^{-\frac{1}{2}}(\A+\I)\D^{-\frac{1}{2}}\right)^{k}\X\W,
\end{equation}
where $\D$ is the degree matrix of $\A+\I$, $\W \in \mathbb{R}^{C \times F}$ is a learnable parameter matrix, and $k$ is the order of the normalized adjacency matrix.

To study the influence of different orders of the adjacency matrix, we adopt five orders from $1$ to $5$ to train five models on the node classification task.
Fig.~\ref{fig:order} presents the node classification accuracy under different orders of the adjacency matrix for the proposed method and DGI respectively.
As we can see, the proposed method achieves best classification performance when $k=\{4,2,3\}$ on the three datasets respectively, and outperforms GAE in different orders.
When $k=1$, our model still achieves reasonable results although it is difficult to predict the topology transformations from $1$-hop neighborhood information;
when $k>1$, our model outperforms DGI by a large margin on Cora and Pubmed dataset, and achieves comparable results to DGI on Citeseer dataset.
This is because DGI adopts feature shuffling to generate negative samples, which is insufficient to learn contrastive feature representations when aggregating multi-hop neighborhood information, while the proposed method takes advantage of multi-hop neighborhood information to predict the topology transformations, leading to improved performance.

\begin{figure}[t]
  \centering
  \includegraphics[width=0.7\columnwidth]{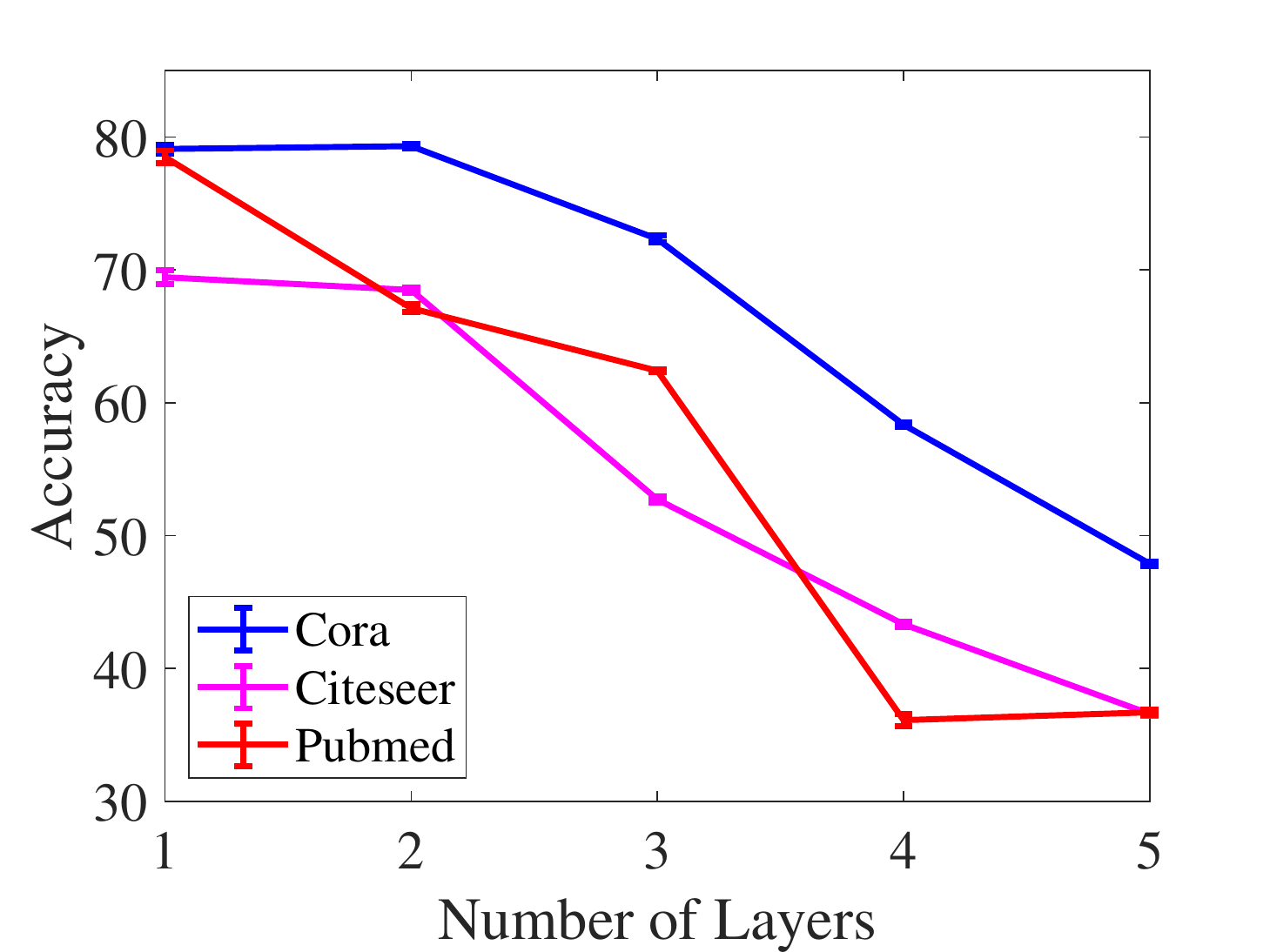}
  \caption{Node classification accuracies under different numbers of SGC layers on the Cora, Citeseer, and Pubmed datasets.}
  \label{fig:num_layers}
\end{figure}

\subsubsection{Experiments on Different Numbers of SGC Layers}

We conduct ablation studies on the number of SGC layers to study the influence of deeper encoder architecture on the node classification task.
In this experiment, we adopt an encoder with five numbers of SGC layers ranging from $1$ to $5$ to train five models on the {\it Cora}, {\it Citeseer}, and {\it Pubmed} datasets.
To prevent memory issues during the training, the channel of each hidden layer is set to $64$, and the output channel is $512$.
Fig.~\ref{fig:num_layers} presents the node classification accuracy under different numbers of SGC layers for the proposed method.
We see that, when the number of layers is small (\textit{e.g.}, $1$ or $2$), our model achieves reasonable results on the three datasets;
when the number of SGC layers increases, the performance of our model decreases quickly due to {\it overfitting}.
This is consistent with the conclusion in \cite{kipf2017semi} that overfitting can become an issue as the number of parameters increases with model depth in graph convolutional neural networks.

\begin{table*}[t]
\centering
\caption{Graph classification accuracies (with standard deviation) in percentage on $6$ datasets. 
``\textgreater 1 Day" represents that the computation exceeds 24 hours. ``OOM" is out of memory error.}
\label{tab:results_gc}
\begin{tabular}{lcccccc}
\hline
\multicolumn{1}{l|}{\textbf{Dataset}} & MUTAG & PTC-MR & RDT-B & RDT-M5K & IMDB-B & IMDB-M \\
\multicolumn{1}{l|}{\textbf{(\#Graphs)}} & 188 & 344 & 2000 & 4999 & 1000 & 1500 \\
\multicolumn{1}{l|}{\textbf{(\#Classes)}} & 2 & 2 & 2 & 5 & 2 & 3 \\
\multicolumn{1}{l|}{\textbf{(Avg. \#Nodes)}} & 17.93 & 14.29 & 429.63 & 508.52 & 19.77 & 13.00 \\
\hline
\multicolumn{7}{c}{\textbf{Graph Kernel Methods}} \\ \hline
\multicolumn{1}{l|}{RW} & $83.72 \pm 1.50$ & $57.85 \pm 1.30$ & OOM & OOM & $50.68 \pm 0.26$ & $34.65 \pm 0.19$ \\
\multicolumn{1}{l|}{SP} & $85.22 \pm 2.43$ & $58.24 \pm 2.44$ & $64.11 \pm 0.14$ & $39.55 \pm 0.22$ & $55.60 \pm 0.22$ & $37.99 \pm 0.30$ \\
\multicolumn{1}{l|}{GK} & $81.66 \pm 2.11$ & $57.26 \pm 1.41$ & $77.34 \pm 0.18$ & $41.01 \pm 0.17$ & $65.87 \pm 0.98$ & $43.89 \pm 0.38$ \\
\multicolumn{1}{l|}{WL} & $80.72 \pm 3.00$ & $57.97 \pm 0.49$ & $68.82 \pm 0.41$ & $46.06 \pm 0.21$ & $72.30 \pm 3.44$ & $46.95 \pm 0.46$ \\
\multicolumn{1}{l|}{DGK} & $87.44 \pm 2.72$ & $60.08 \pm 2.55$ & $78.04 \pm 0.39$ & $41.27 \pm 0.18$ & $66.96 \pm 0.56$ & $44.55 \pm 0.52$ \\
\multicolumn{1}{l|}{MLG} & $87.94 \pm 1.61$ & $63.26 \pm 1.48$ & \textgreater 1 Day & \textgreater 1 Day & $66.55 \pm 0.25$ & $41.17 \pm 0.03$ \\ \hline
\multicolumn{7}{c}{\textbf{Supervised Methods}} \\ \hline
\multicolumn{1}{l|}{GCN} & $85.6 \pm 5.8$ & $64.2 \pm 4.3$ & $50.0 \pm 0.0$ & $20.0 \pm 0.0$ & $74.0 \pm 3.0$ & $51.9 \pm 3.8$ \\
\multicolumn{1}{l|}{GraphSAGE} & $85.1 \pm 7.6$ & $63.9 \pm 7.7$ & - & - & $72.3 \pm 5.3$ & $50.9 \pm 2.2$ \\
\multicolumn{1}{l|}{GIN-0} & $89.4 \pm 5.6$ & $64.6 \pm 7.0$ & $92.4 \pm 2.5$ & $57.5 \pm 1.5$ & $75.1 \pm 5.1$ & $52.3 \pm 2.8$ \\
\multicolumn{1}{l|}{GIN-$\epsilon$} & $89.0 \pm 6.0$ & $63.7 \pm 8.2$ & $92.2 \pm 2.3$ & $57.0 \pm 1.7$ & $74.3 \pm 5.1$ & $52.1 \pm 3.6$ \\
\multicolumn{1}{l|}{NGN} & $89.4 \pm 1.6$ & $66.8 \pm 1.8$ & - & - & $74.8 \pm 2.0$ & $51.3 \pm 1.5$ \\ \hline
\multicolumn{7}{c}{\textbf{Unsupervised Methods}} \\ \hline
\multicolumn{1}{l|}{node2vec} & $72.63 \pm 10.20$ & $58.58 \pm 8.00$ & - & - & - & - \\
\multicolumn{1}{l|}{sub2vec} & $61.05 \pm 15.80$ & $59.99 \pm 6.38$ & $71.48 \pm 0.41$ & $36.68 \pm 0.42$ & $55.26 \pm 1.54$ & $36.67 \pm 0.83$ \\
\multicolumn{1}{l|}{graph2vec} & $83.15 \pm 9.25$ & $60.17 \pm 6.86$ & $75.78 \pm 1.03$ & $47.86 \pm 0.26$ & $71.10 \pm 0.54$ & $\mathbf{50.44 \pm 0.87}$ \\
\multicolumn{1}{l|}{InfoGraph} & $89.01 \pm 1.13$ & $61.65 \pm 1.43$ & $82.50 \pm 1.42$ & $53.46 \pm 1.03$ & $73.03 \pm 0.87$ & $49.69 \pm 0.53$ \\
\multicolumn{1}{l|}{\textbf{Ours}} & $\mathbf{89.25 \pm 0.81}$ & $\mathbf{64.59 \pm 1.26}$ & $\mathbf{84.93 \pm 0.18}$ & $\mathbf{55.52 \pm 0.20}$ & $\mathbf{73.46 \pm 0.38}$ & $49.68 \pm 0.31$ \\ \hline
\end{tabular}
\end{table*}

\subsubsection{Robustness Test}

To evaluate the robustness of our model on the node classification task, we jitter the original node features with an additive noise model, namely,
\begin{equation}
  \tX=\X+\E,
\end{equation}
where $\X\in\mathbb{R}^{N\times F}$ is the original node features, $\E\in\mathbb{R}^{N\times F}$ is a random matrix which is sampled from a random distribution (\eg, Gaussian or Laplace), and $\tX$ denotes the noise-corrupted node features.

Specifically, we select the zero-mean Gaussian noise with a range of standard deviation $\sigma$ from $0.01$ to $0.10$ at an interval of $0.01$, as well as the zero-location Laplace noise with a range of scale parameter $s$ from $0.01$ to $0.10$ at an interval of $0.01$, for extensive classification performance comparison.
We employ one SGC layer as the encoder $E(\cdot)$ of the proposed model and two representative self-supervised model DGI and GAE, where the order of the adjacency matrix is set to $1$.
We use the original node features $\X$ to train the three models in the unsupervised training stage and the linear classifier in the supervised evaluation stage, and employ the corrupted node features $\tX$ to evaluate the classification accuracies.

The classification performance under Gaussian and Laplace noises are presented in Fig.~\ref{fig:gaussian_noise} and Fig.~\ref{fig:laplace_noise}, respectively.
When the noise level is low, our model outperforms GAE by a large margin on the Cora and Citeseer datasets, and achieves comparable results to DGI on the three datasets.
When the noise level is high, our model significantly outperforms GAE and DGI on the three datasets.
This is because DGI takes the original graph topology to aggregate node features.
When the original node features are seriously corrupted (the noise level is high), the aggregated features will change significantly.
GAE reconstructs the adjacency matrix from the feature representations of individual nodes. The high dependency of the graph topology and node features leads to bad performance when the node features suffer from serious noise perturbations.
In contrast, our model aims to predict the topology transformations from the feature representations of nodes before and after transformation, which not only employs the original graph topology information, but also explores how node features would change by applying a topology transformation, thus enhancing the robustness.

\begin{figure*}[t]
  \centering
  \subfigure[Cora]{
  \includegraphics[width=0.3\textwidth]{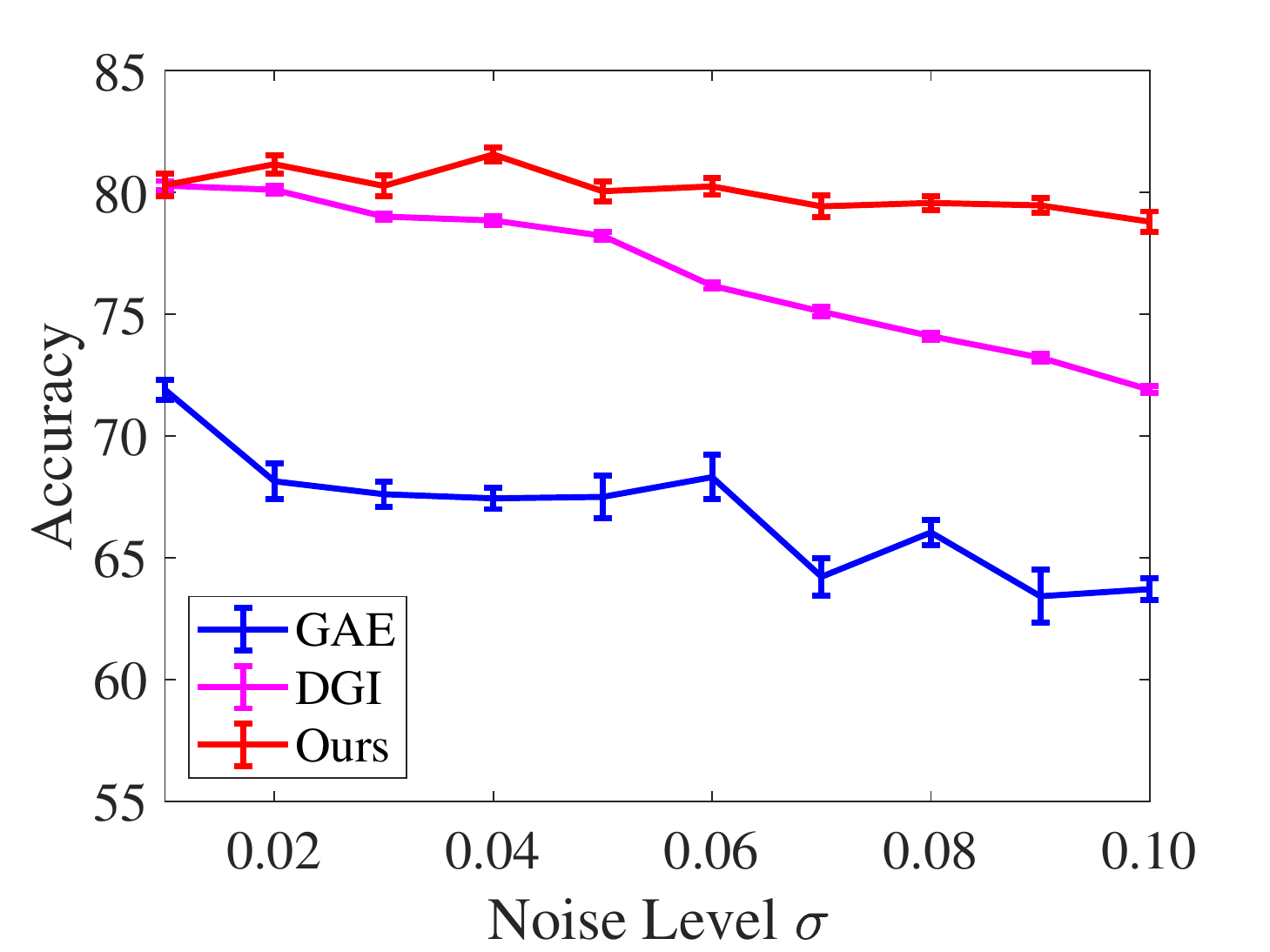}
  \label{subfig:cora_gaussian}
  }
  \subfigure[Citeseer]{
  \includegraphics[width=0.3\textwidth]{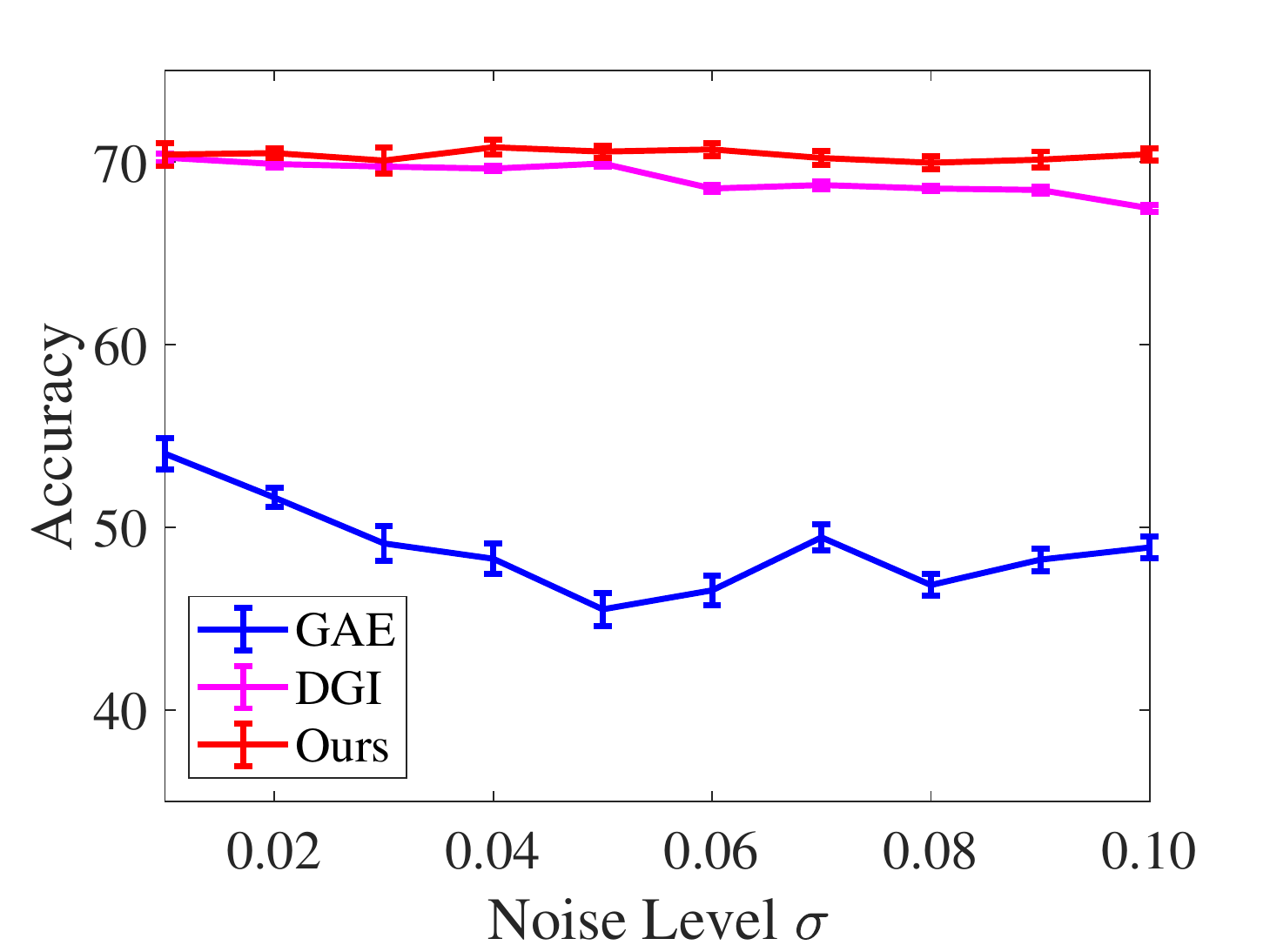}
  \label{subfig:citeseer_gaussian}
  }
  \subfigure[Pubmed]{
  \includegraphics[width=0.3\textwidth]{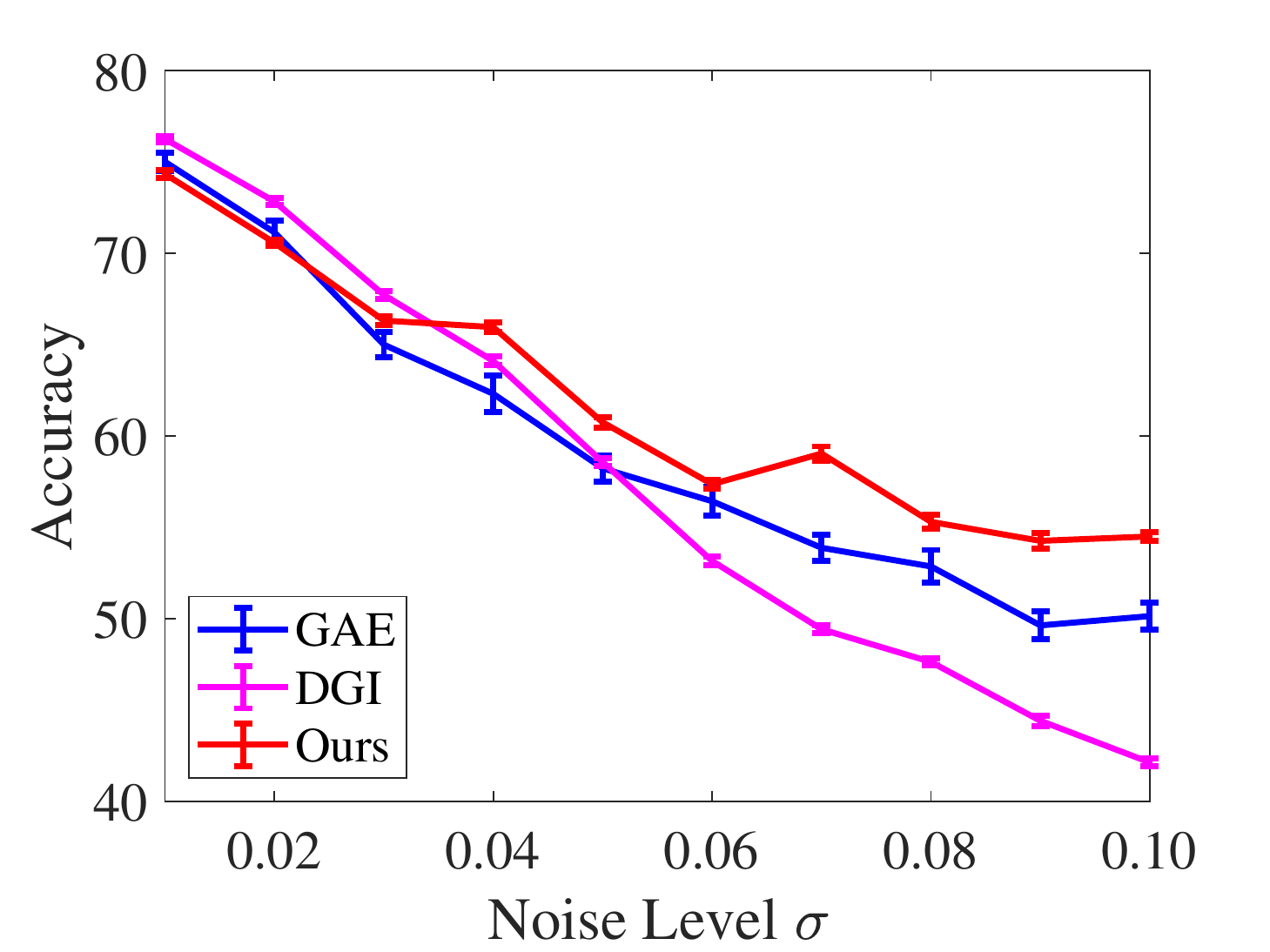}
  \label{subfig:pubmed_gaussian}
  }
  \caption{\textbf{Node classification accuracies under different levels of Gaussian noise on the Cora, Citeseer, and Pubmed datasets.} The horizontal axis and vertical axis represent different noise levels and classification accuracies, receptively.}
  \label{fig:gaussian_noise}
\end{figure*}

\begin{figure*}[t]
  \centering
  \subfigure[Cora]{
  \includegraphics[width=0.3\textwidth]{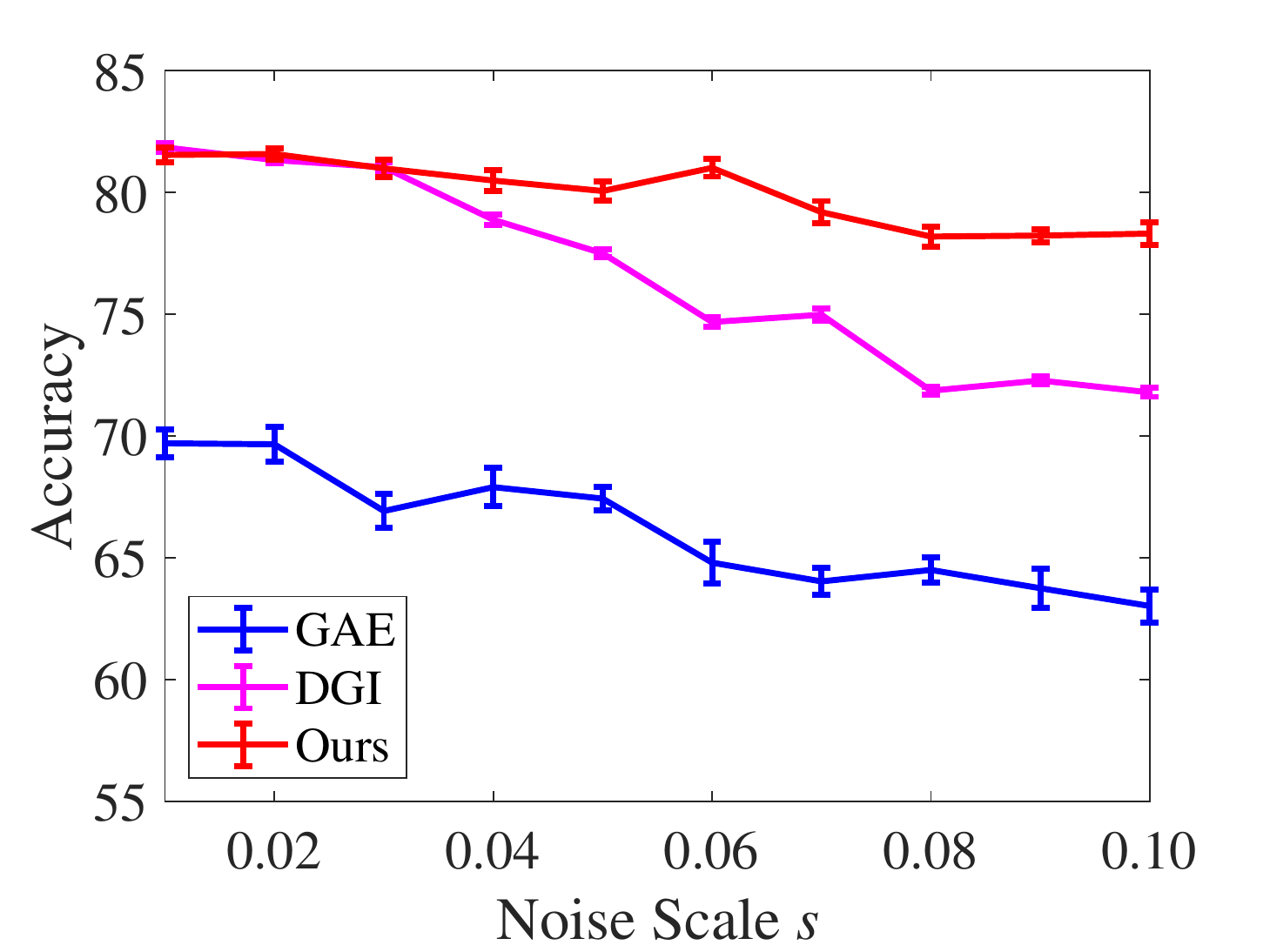}
  \label{subfig:cora_laplace}
  }
  \subfigure[Citeseer]{
  \includegraphics[width=0.3\textwidth]{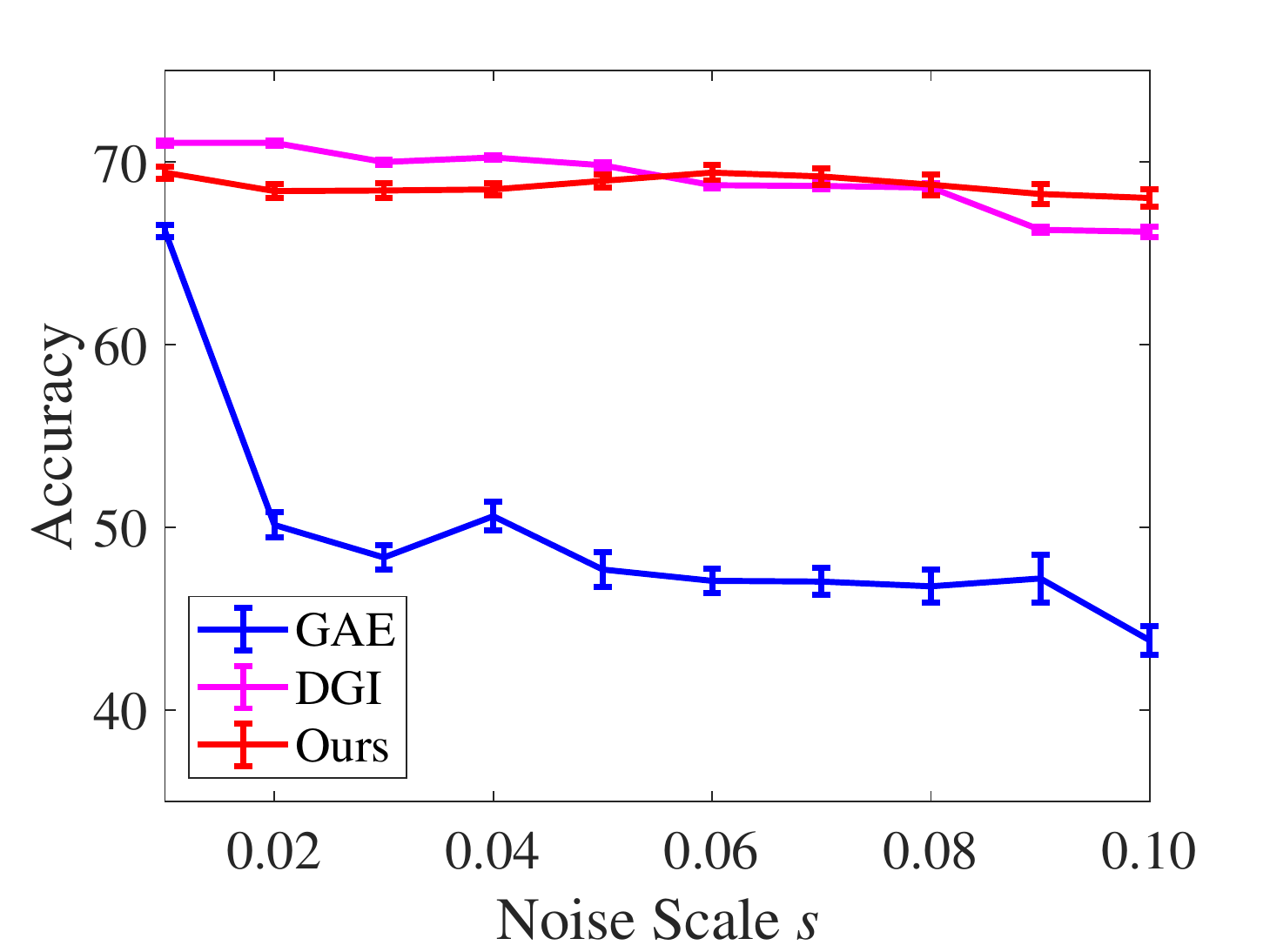}
  \label{subfig:citeseer_laplace}
  }
  \subfigure[Pubmed]{
  \includegraphics[width=0.3\textwidth]{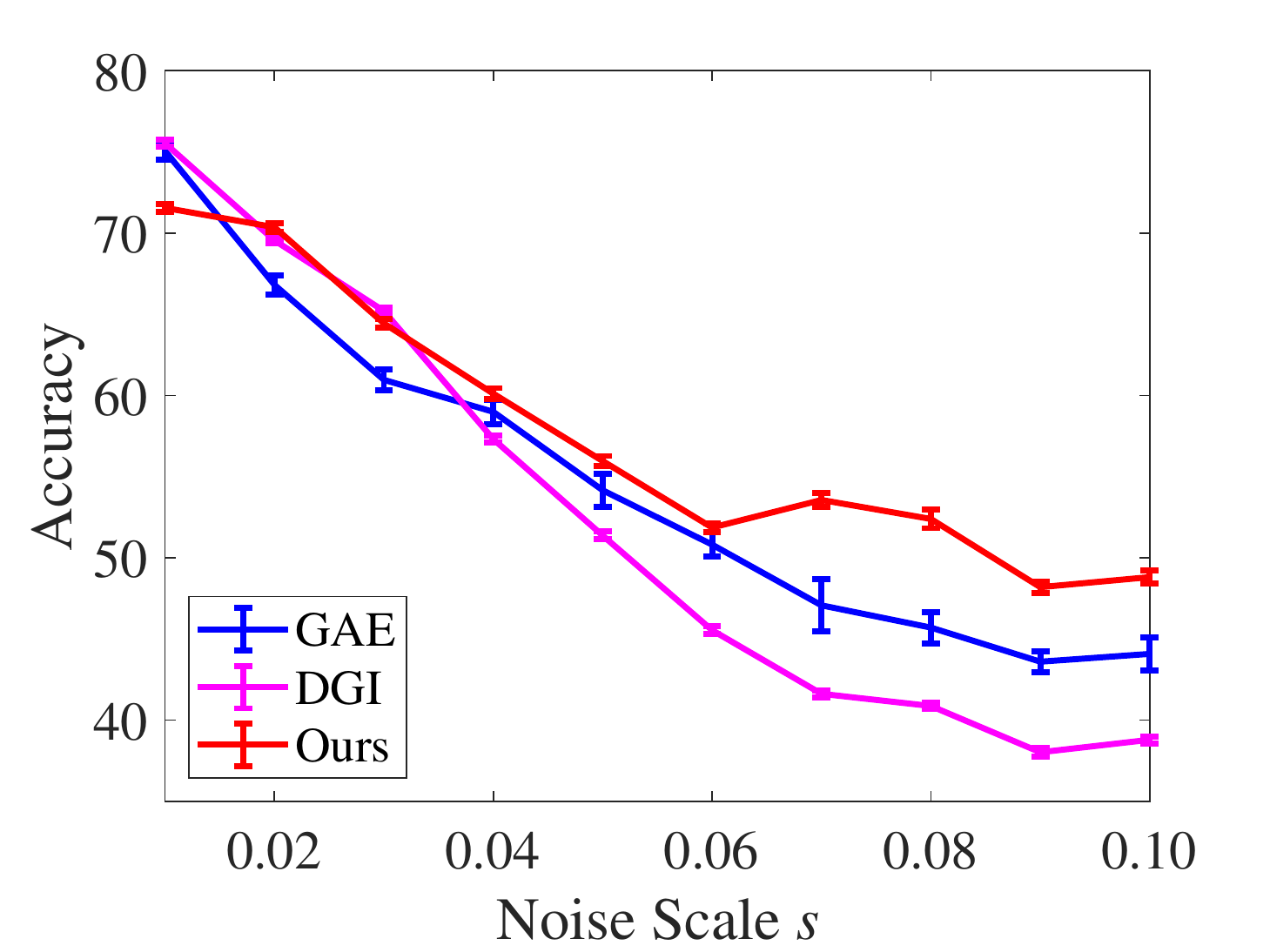}
  \label{subfig:pubmed_laplace}
  }
  \caption{\textbf{Node classification accuracies under different scales of Laplace noise on the Cora, Citeseer, and Pubmed datasets.} The horizontal axis and vertical axis represent different noise levels and classification accuracies, receptively.}
  \label{fig:laplace_noise}
\end{figure*}

\subsubsection{Experiments On Different Edge Perturbation Rates}
\label{app:perturb_rate}

Further, we evaluate the influence of the edge perturbation rate in Eq.~(\ref{eq:split}) on the node classification task.
We choose $11$ edge perturbation rates from $0.0$ to $1.0$ at an interval of $0.1$ to train the proposed model.
We use one SGC layer as our encoder $E(\cdot)$, where the order of the adjacency matrix is set to $1$.
As presented in Fig.~\ref{fig:perturbation_rate}, the blue solid line with error bar shows the classification accuracy of our method under different edge perturbation rates.
We also provide the classification accuracy on feature representations of graphs from a randomly initialized encoder $E(\cdot)$, denoted as \textit{Random Init.}, which serves as the lower bound of the performance. 

\begin{figure*}[t]
  \centering
  \subfigure[Cora]{
  \includegraphics[width=0.3\textwidth]{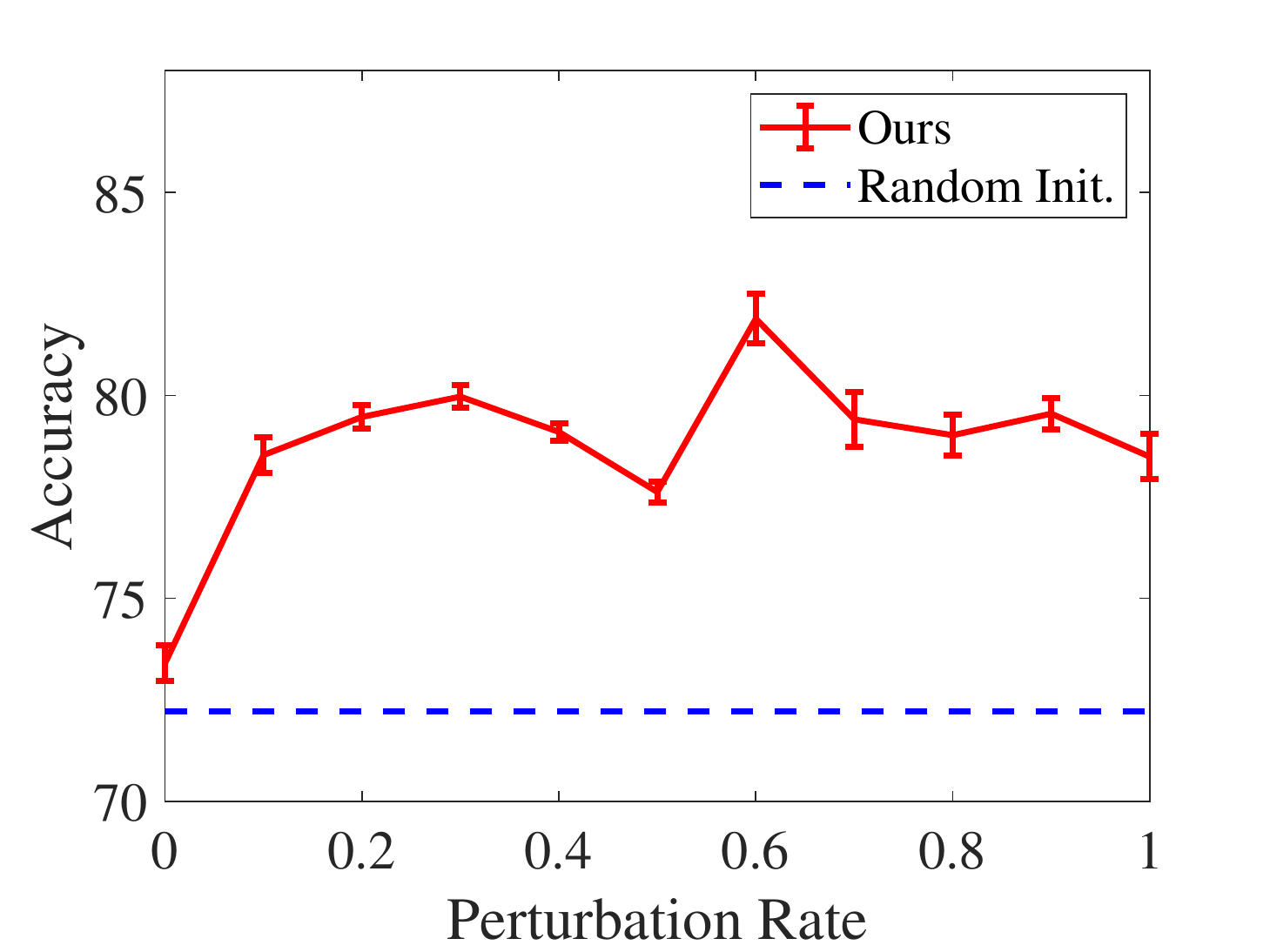}
  \label{subfig:cora}
  }
  \subfigure[Citeseer]{
  \includegraphics[width=0.3\textwidth]{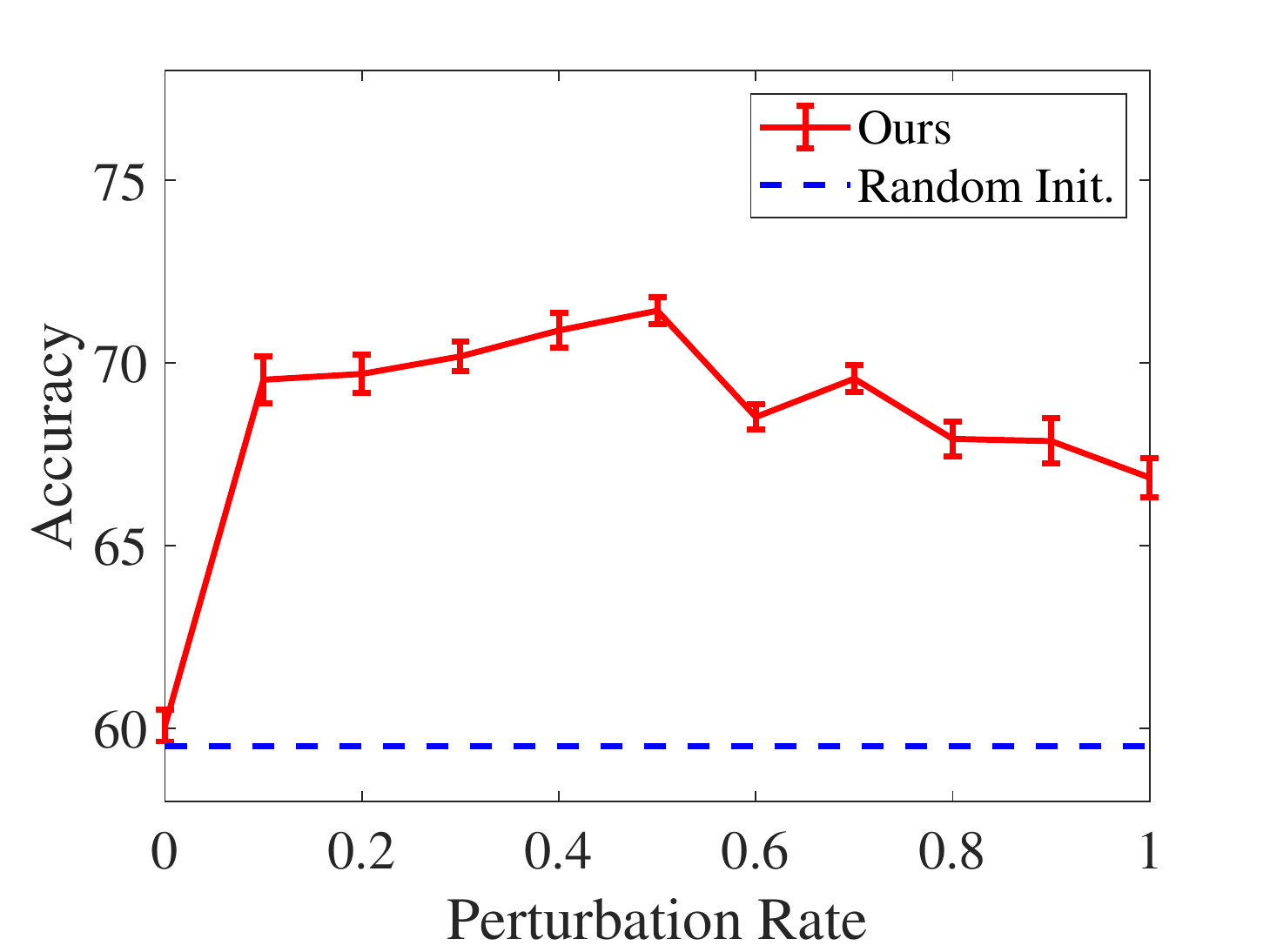}
  \label{subfig:citeseer}
  }
  \subfigure[Pubmed]{
  \includegraphics[width=0.3\textwidth]{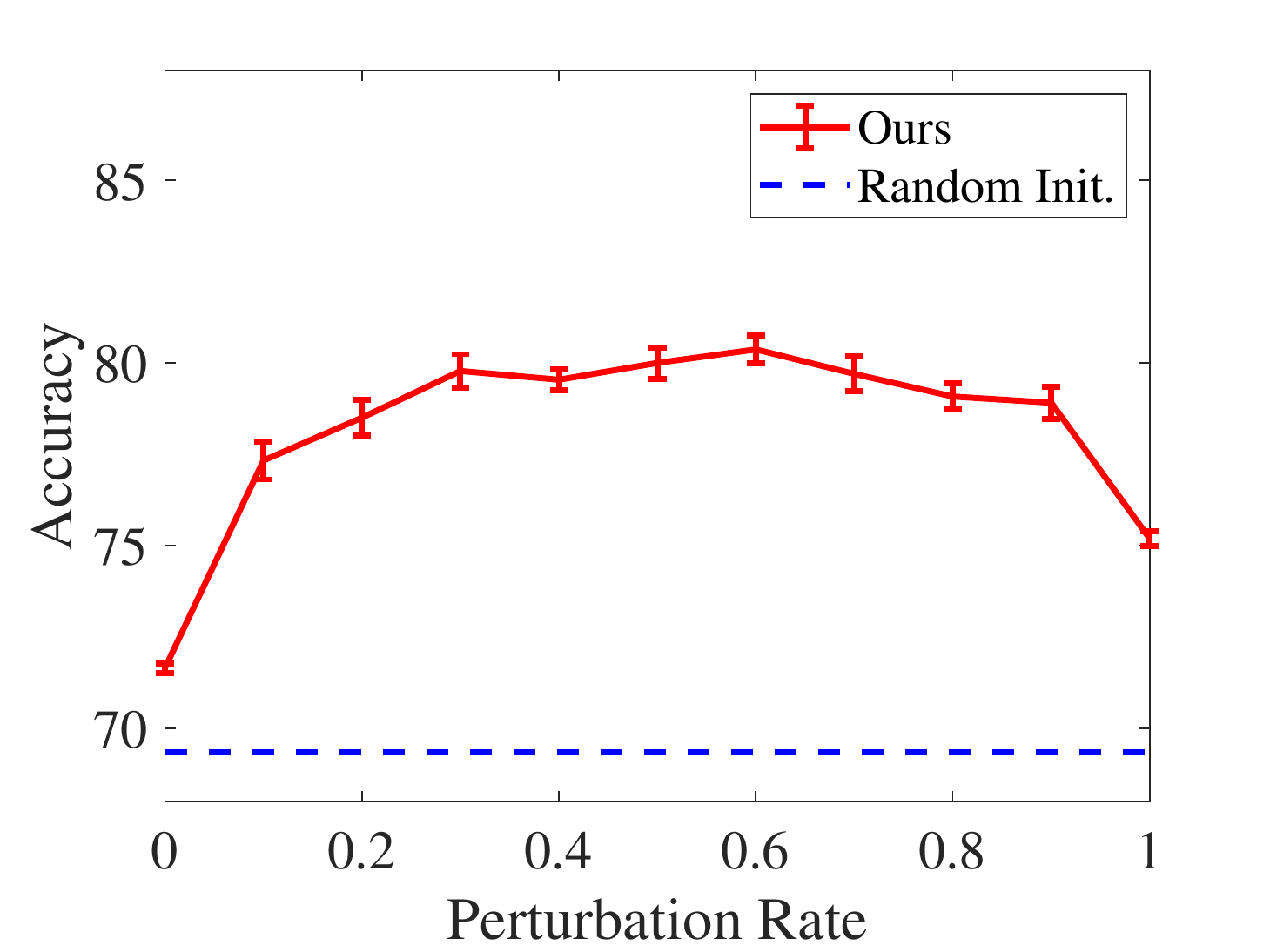}
  \label{subfig:pubmed}
  }
  \caption{\textbf{Node classification accuracies under different edge perturbation rates on the Cora, Citeseer, and Pubmed datasets.} The horizontal axis and vertical axis represent perturbation rates and classification accuracies, respectively.}
  \label{fig:perturbation_rate}
\end{figure*}

As we can see, the classification performance reaches the best when the graph is perturbed under a reasonable edge perturbation rate, \eg, $r=\{0.6,0.5,0.6\}$ for the Cora, Citeseer, and Pubmed dataset, respectively.
When the edge perturbation rate $r=0.0$, the unsupervised training task of the proposed model becomes link prediction, which cannot take advantage of the proposed method by predicting the topology transformations;
when the edge perturbation rate $r=1.0$, our {model} still achieves reasonable classification results, which shows the stability of our model under high edge perturbation rates.
At the same time, we observe that the proposed {method} outperforms \textit{Random Init.} by a large margin, which validates the effectiveness of the proposed unsupervised training strategy.

\subsection{Graph Classification}

\subsubsection{Datasets} 
We conduct graph classification on six well-known graph benchmark datasets \cite{yanardag2015deep}, including two molecule datasets MUTAG and PTC, and four social network datasets REDDIT-BINARY, REDDIT-MULTI-5K, IMDB-BINARY, and IMDB-MULTI.
In the two molecule datasets, graphs are molecules, where nodes represent atoms and edges represent chemical bonds, and the graph classification task is to classify the molecules.
In the REDDIT dataset, a graph denotes a discussion thread, where nodes correspond to users, two of which are connected by an edge if one responded to a comment of the other.
The graph classification task is to distinguish whether the subreddits is discussion-based or question/answer-based (REDDIT-BINARY), or predict the subreddit (REDDIT-MULTI-5K).
The IMDB dataset consists of ego-networks derived from actor collaborations, and the graph classification task is to predict the genre, \eg, Action or Romance.

\subsubsection{Implementation Details} 
In this task, the entire network is trained via Adam optimizer with a batch size of $64$, and the learning rate is set to $10^{-3}$.
For the encoder architecture, we follow the same encoder settings in the released code of InfoGraph \cite{sun2020infograph}, \ie, three Graph Isomorphism Network (GIN) layers \cite{xu2018powerful} with batch normalization.
We also use one linear layer to classify the transformation types.
We set the sampling rate $r=0.5$ for all datasets.

During the evaluation stage, the entire encoder will be frozen to extract node-level feature representations, which will go through a global add pooling layer to acquire global features.
We then use LIBSVM to classify these global features to classification scores.
We adopt the same procedure of previous works \cite{sun2020infograph} to make a fair comparison and use 10-fold cross validation accuracy to report the classification performance, and the experiments are repeated five times.

\subsubsection{Experimental Results} 
We take six graph kernel approaches for comparison: Random Walk (RW) \cite{gartner2003graph}, Shortest Path Kernel (SP) \cite{borgwardt2005shortest}, Graphlet Kernel (GK) \cite{shervashidze2009efficient}, Weisfeiler-Lehman Sub-tree Kernel (WL) \cite{shervashidze2011weisfeiler}, Deep Graph Kernels (DGK) \cite{yanardag2015deep}, and Multi-Scale Laplacian Kernel (MLG) \cite{kondor2016multiscale}.
Aside from graph kernel methods, we also compare with three unsupervised graph-level representation learning methods: node2vec \cite{grover2016node2vec}, sub2vec \cite{adhikari2018sub2vec}, and graph2vec \cite{narayanan2017graph2vec}, and one contrastive learning method: InfoGraph \cite{sun2020infograph}.
The experimental results of unsupervised graph classification are preseted in Tab.~\ref{tab:results_gc}.
The proposed method outperforms all unsupervised baseline methods on the first five datasets, and achieves comparable results on the other dataset. 
Also, the proposed approach reaches the performance of supervised methods at times, thus validating the superiority of the {proposed method}.

\begin{table*}[t]
\centering
\caption{Link prediction results (with standard deviation) in percentage on three datasets.}
\label{tab:link_pred}
\begin{tabular}{l|cccccc}
\hline
\multicolumn{1}{c|}{\multirow{2}{*}{\textbf{Method}}} & \multicolumn{2}{c}{\textbf{Cora}} & \multicolumn{2}{c}{\textbf{Citeseer}} & \multicolumn{2}{c}{\textbf{Pubmed}} \\
\multicolumn{1}{c|}{} & \textbf{AUC} & \textbf{AP} & \textbf{AUC} & \textbf{AP} & \textbf{AUC} & \textbf{AP} \\ \hline
Spectral Clustering \cite{tang2011leveraging} & $84.6 \pm 0.01$ & $88.5 \pm 0.00$ & $80.5 \pm 0.01$ & $85.0 \pm 0.01$ & $84.2 \pm 0.02$ & $87.7 \pm 0.01$ \\
DeepWalk \cite{perozzi2014deepwalk} & $83.1 \pm 0.01$ & $85.0 \pm 0.00$ & $80.5 \pm 0.02$ & $83.6 \pm 0.01$ & $84.2 \pm 0.00$ & $84.1 \pm 0.00$ \\
GAE \cite{kipf2016variational} & $91.0 \pm 0.02$ & $92.0 \pm 0.03$ & $89.5 \pm 0.04$ & $89.9 \pm 0.05$ & $96.4 \pm 0.00$ & $96.5 \pm 0.00$ \\
VGAE \cite{kipf2016variational} & $91.4 \pm 0.01$ & $92.6 \pm 0.01$ & $90.8 \pm 0.02$ & $92.0 \pm 0.02$ & $94.4 \pm 0.02$ & $94.7 \pm 0.02$ \\
CensNet-VAE \cite{jiang2020co} & $91.7 \pm 0.02$ & $92.6 \pm 0.01$ & $90.6 \pm 0.01$ & $91.6 \pm 0.01$ & $95.5 \pm 0.03$ & $95.9 \pm 0.02$ \\ \hline
\textbf{Ours} & $\mathbf{93.4 \pm 1.15}$ & $\mathbf{92.7 \pm 1.16}$ & $\mathbf{92.7 \pm 0.50}$ & $\mathbf{91.9 \pm 0.74}$ & $\mathbf{96.8 \pm 0.70}$ & $\mathbf{96.8 \pm 0.70}$ \\ \hline
\end{tabular}
\end{table*}

\subsection{Link Prediction on Static Graphs}

\subsubsection{Implementation Details}

Also, we evaluate our model on link prediction over the three citation network datasets Cora, Citeseer, and Pubmed.
We follow the same experimental settings in \cite{kipf2016variational}, where $85\%$ of citation links are used for training, and $10\%$ and $5\%$ of citation links are for testing and validating respectively.
We employ the same number of randomly sampled pairs of disconnected nodes for testing.

We deploy two GCN \cite{kipf2017semi} layers with $32$ hidden channels and $16$ output channels as the encoder as in \cite{kipf2016variational}.
The ReLU activation function is employed between the two GCN layers.
Subsequent to the encoder, we also use one linear layer to classify the transformation types.

In this experiment, we demonstrate the benefits of pre-training with our method before training on link prediction.
During the training stage, we first train our model by minimizing Eq.~(\ref{eq:loss}) on the training set.
After that, we replace the transformation decoder with an inner product decoder to acquire the reconstructed adjacency matrix $\widehat{\A}$, \ie,
\begin{equation}
  \widehat{\A}=\mathrm{sigmoid}\left(\H\H^{\top}\right), \; \text{where} \; \H=E(\X,\A),
\end{equation}
where $\A$ is the original adjacency matrix, and  $\mathrm{sigmoid}(\cdot)$ is an activation function.
The encoder is thus fine-tuned by minimizing the reconstruction error between $\A$ and $\widehat{\A}$.
During the evaluation stage, we predict the edges in the testing set from the feature representations $\H$, and report \textit{area under the ROC curve} (AUC) and \textit{average precision} (AP) scores for each dataset.

\subsubsection{Experimental Results}

We take five approaches for comparison: Spectral Clustering \cite{tang2011leveraging}, DeepWalk \cite{perozzi2014deepwalk}, GAE \cite{kipf2016variational}, VGAE \cite{kipf2016variational}, and CensNet-VAE \cite{jiang2020co}.
We present the mean AUC and AP (with standard deviation) after $10$ runs of training with random initialization on fixed dataset splits.
As reported in Tab.~\ref{tab:link_pred}, the proposed method outperforms all AE-based methods on the three datasets.
This is because our method makes better use of the topology information of graphs by decoding the topology perturbation, thereby revealing not only static visual structures but also how they would change by applying different topology transformations.
This again validates the effectiveness of our method.

\begin{table*}[t]
\scriptsize
\centering
\caption{Link prediction results (with standard deviation) in percentage on three temporal graph datasets.}
\label{tab:dynamic_lp}
\begin{tabular}{l|cccccc}
\hline
\multicolumn{1}{c|}{\multirow{2}{*}{\textbf{Method}}} & \multicolumn{2}{c}{\textbf{Sparrow Social}} & \multicolumn{2}{c}{\textbf{Ant Colony}} & \multicolumn{2}{c}{\textbf{Wildbird Network}} \\
\multicolumn{1}{c|}{} & \textbf{AUC} & \textbf{AP} & \textbf{AUC} & \textbf{AP} & \textbf{AUC} & \textbf{AP} \\ \hline
GAE & \multicolumn{1}{l}{$56.5 \pm 0.54$} & \multicolumn{1}{l}{$52.1 \pm 0.26$} & \multicolumn{1}{l}{$59.1 \pm 5.87$} & \multicolumn{1}{l}{$51.9 \pm 3.27$} & \multicolumn{1}{l}{$52.8 \pm 0.30$} & \multicolumn{1}{l}{$47.9 \pm 0.33$} \\
VGAE & $55.0 \pm 0.17$ & $52.1 \pm 0.26$ & $66.0 \pm 0.29$ & $56.0 \pm 0.44$ & $52.3 \pm 0.25$ & $47.5 \pm 0.38$ \\
\textbf{Ours} & $\mathbf{67.7 \pm 4.19}$ & $\mathbf{59.6 \pm 2.51}$ & $\mathbf{70.0 \pm 0.30}$ & $\mathbf{58.5 \pm 0.30}$ & $\mathbf{60.8 \pm 4.19}$ & $\mathbf{54.9 \pm 2.68}$ \\ \hline
\end{tabular}
\end{table*}

\subsection{Link Prediction on Temporal Graphs}

\subsubsection{Implementation Details}

In order to validate the intuition of the benefit of our model for graph data changing over time, we further perform link prediction over temporal graph datasets.
Since most existing methods focus on static graphs, few methods have been proposed to study link prediction on dynamic graphs.
For example, DeepWalk \cite{perozzi2014deepwalk} generates node representations by randomly walking on static graphs; graph auto-encoders \cite{kipf2016variational} learn feature representations of graphs by decoding the input static graphs.
Hence, we propose an experimental setup to make graph auto-encoder methods adaptable to link prediction on temporal graphs.

Considering two consecutive graphs $\mathcal{G}_{t-1}=\{\mathcal{V}_{t-1},\mathcal{E}_{t-1},\A_{t-1}\}$ and $\mathcal{G}_{t}=\{\mathcal{V}_{t},\mathcal{E}_{t},\A_{t}\}$ from a temporal graph sequence $\{\mathcal{G}_{i}\}_{i=1}^{T}$ with $T$ graphs, our goal is to predict the edges in $\mathcal{G}_t$ given $\mathcal{G}_{t-1}$.
For the graph auto-encoder method, we feed the adjacency matrix $\A_{t-1}$ and its corresponding node features $\X_{t-1}$ to the encoder $E(\cdot)$ to learn the feature representations $\H_{t-1}$ at time $t-1$, and reconstruct the adjacency $\widehat{\A}_{t}$ at time $t$ through an inner product decoder, \ie,
\begin{equation}
\begin{split}
  &\widehat{\A}_{t}=\mathrm{sigmoid}\left(\H_{t-1}\H_{t-1}^{\top}\right), \\
  \text{where} \; &\H_{t-1}=E(\X_{t-1},\A_{t-1}),
\end{split}
\end{equation}
where $\mathrm{sigmoid}(\cdot)$ is an activation function.
For our model, we treat $\A_{t-1}$ and $\A_{t}$ as the original and transformed graphs, respectively.
The topology transformation is acquired by simply calculating $\dA=\A_{t}-\A_{t-1}$.
Since these datasets have no initial node features, we set all node features as $1$ to make our model mainly learn feature representations from graph structures, \ie, $\X_{t-1}=\X_{t}=\mathbf{1}_{N\times 1}$.
We thus feed the original and transformed graphs $\{\A_{t-1},\A_{t}\}$ to estimate the topology transformation $\dA$.
After the training of the graph auto-encoder model and our model, we feed the adjacency matrix $\A_{t-1}$ to the encoder with frozen weights to acquire the feature representation $\H_{t-1}$, and predict the edges in $\A_{t}$ through an inner product decoder.

In this experiment, we evaluate our model and two baseline methods GAE and VGAE \cite{kipf2016variational} on link prediction over three real-world temporal graph datasets \cite{rossi2015network}: Sparrow Social, Ant Colony, and Wildbird Network.
We deploy two GCN \cite{kipf2017semi} layers with $32$ hidden channels and $256$ output channels as the encoder.
Subsequent to the encoder, we also use one linear layer to classify the transformation types.
During the evaluation stage, we predict the edges in the testing set from the feature representations $\H_{t-1}$, and report \textit{area under the ROC curve} (AUC) and \textit{average precision} (AP) scores for each dataset.

\subsubsection{Experimental Results}
We present the mean AUC and AP (with standard deviation) after $10$ runs of training with random initialization on fixed dataset splits.
As reported in Tab.~\ref{tab:dynamic_lp}, the proposed method outperforms the two baseline methods on the three datasets. This verifies the effectiveness of our method on temporal graphs.

\begin{figure*}[t]
  \centering
  \subfigure[Node features $\H_{t-1}$ before transformation]{
  \includegraphics[width=0.4\textwidth]{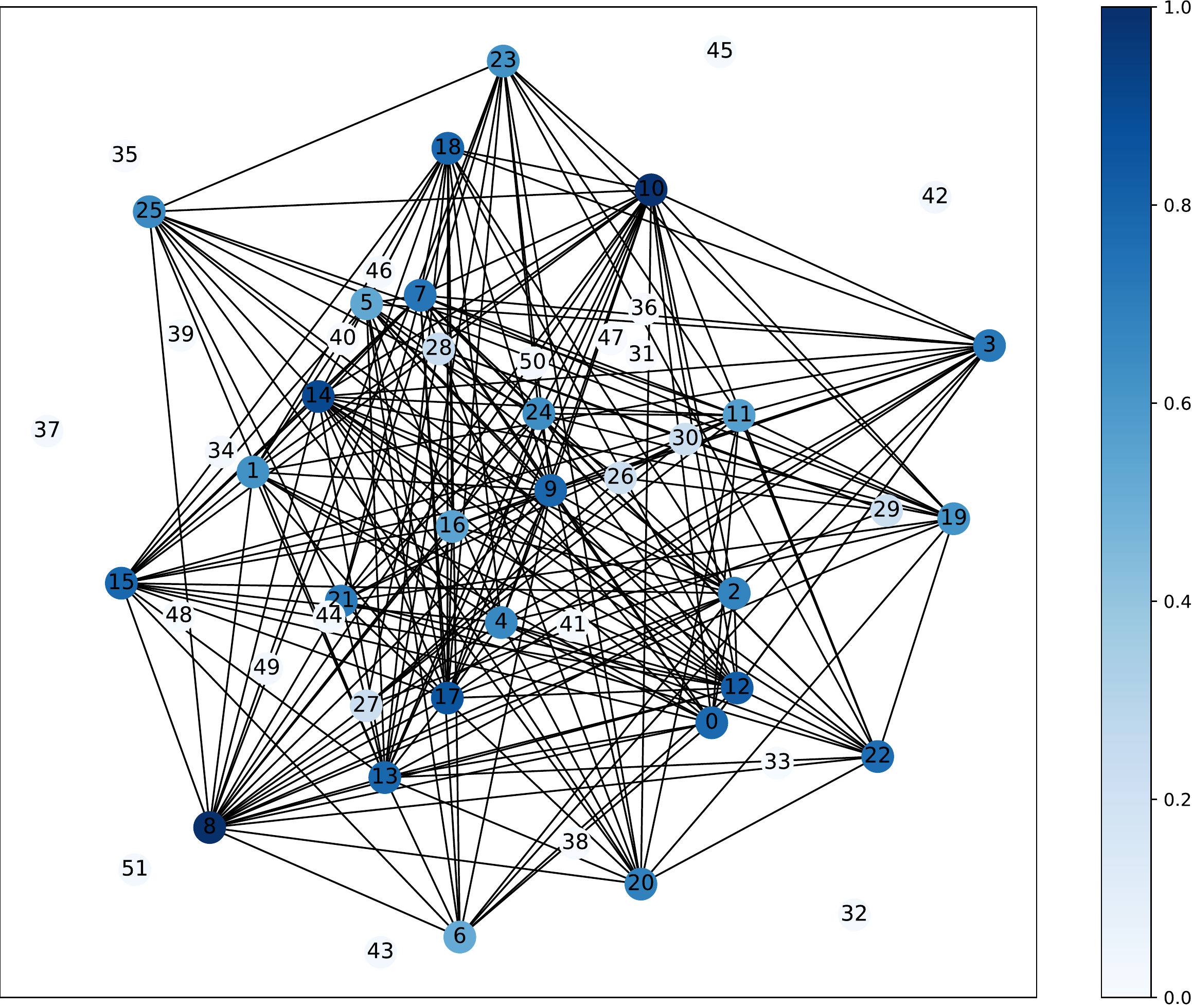}
  \label{subfig:origin_feats}
  }
  \subfigure[Node features $\H_{t}$ after transformation]{
  \includegraphics[width=0.4\textwidth]{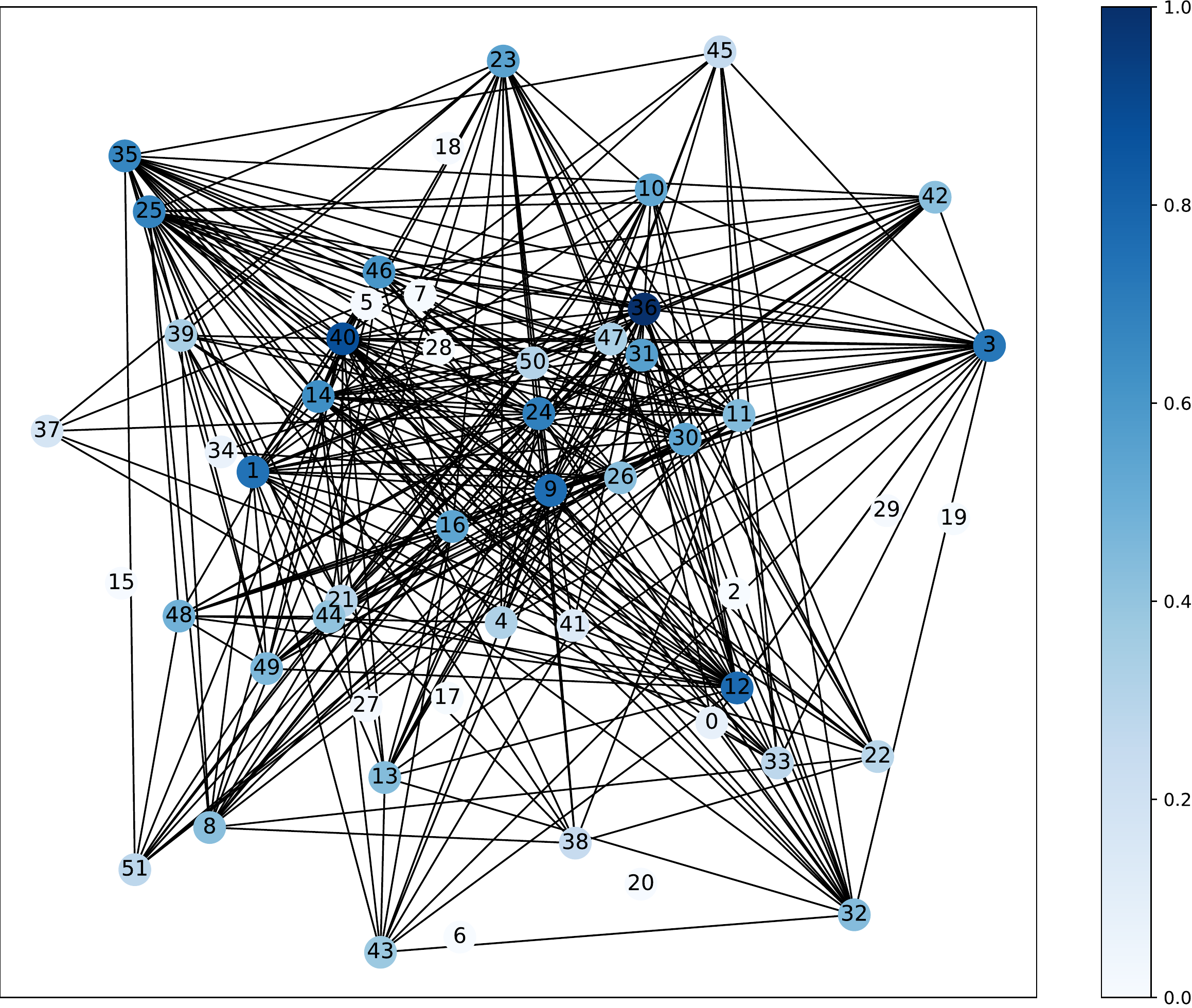}
  \label{subfig:transformed_feats}
  }
  \subfigure[Node features $\widehat{\dH}$ of estimated transformation $\widehat{\dA}$]{
  \includegraphics[width=0.4\textwidth]{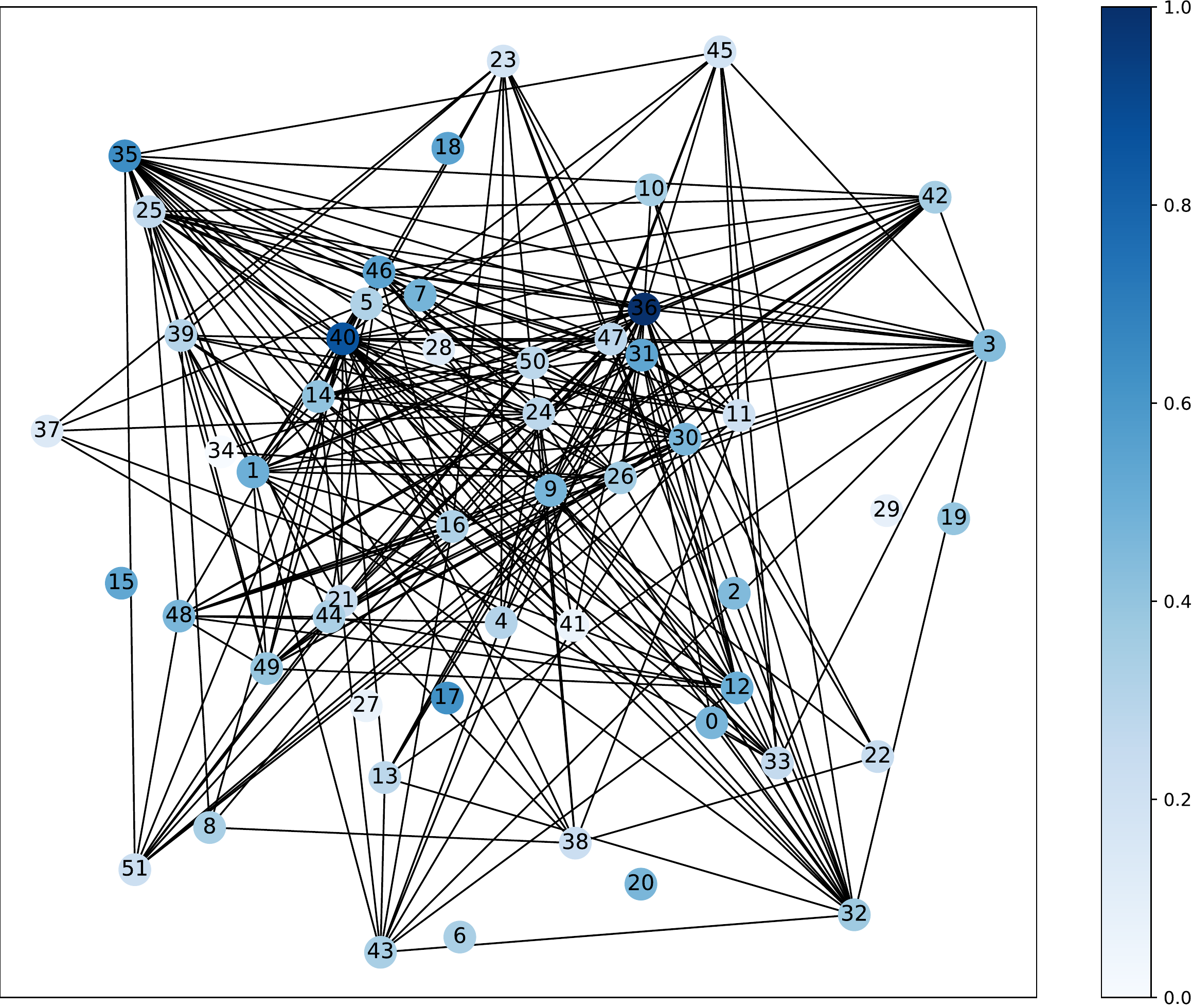}
  \label{subfig:estimated_delta_feats}
  }
  \subfigure[Node features of $\H_{t-1}+\widehat{\dH}$]{
  \includegraphics[width=0.4\textwidth]{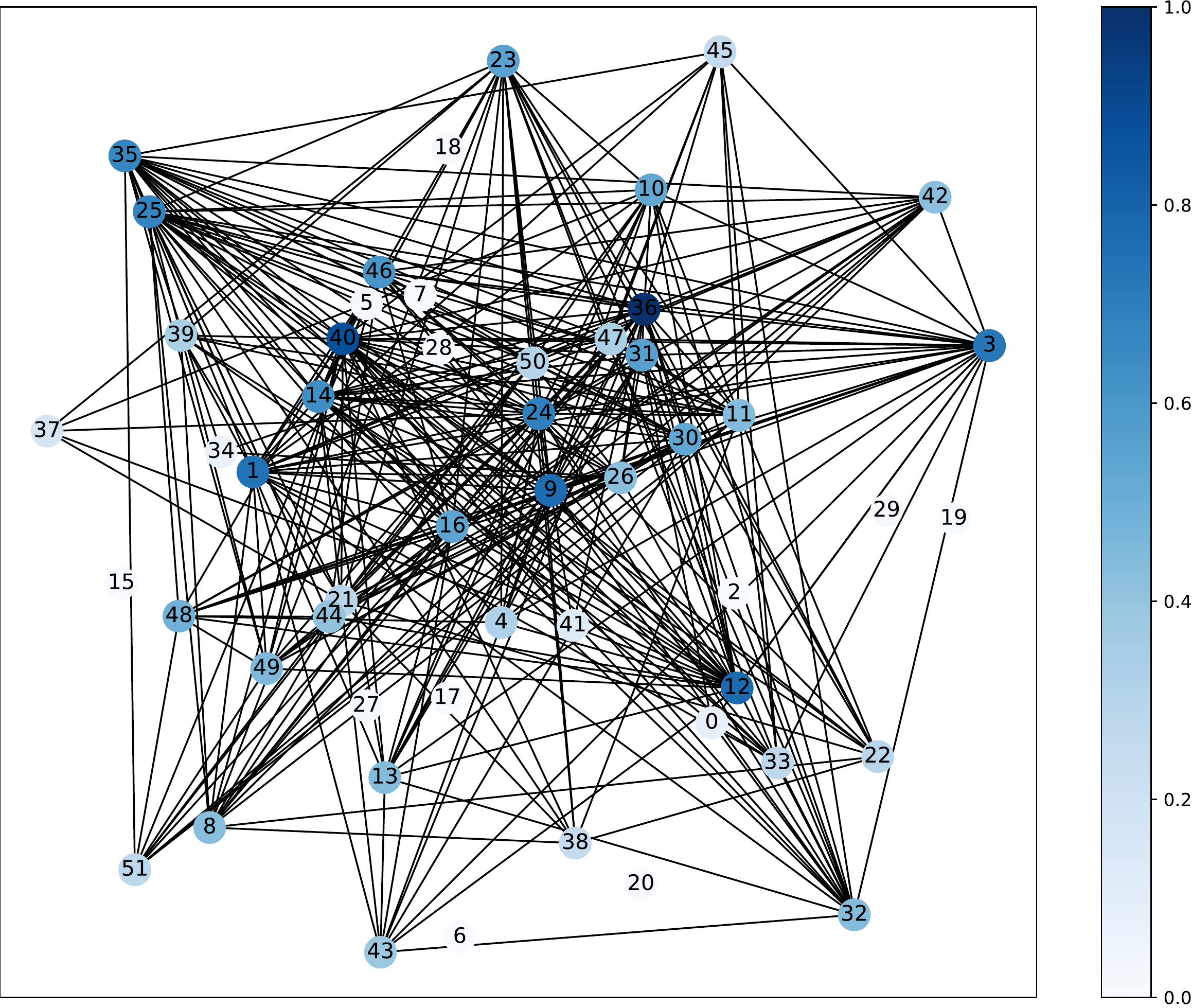}
  \label{subfig:origin_estimated_delta_feats}
  }
  \caption{\textbf{Visualization of node features on the real-scenario temporal graph dataset Sparrow Social \cite{rossi2015network}.} a) The node features learned from the original graph; b) The node features learned from the transformed graph; c) The changed node features learned from the estimated topology transformation; d) The node features of $\H_{t-1}+\widehat{\dH}$. Note that we only visualize the added edges in $\widehat{\dA}$ in (c) for easy observation, while the unchanged and removed edges are ignored. }
  \label{fig:vis_ter}
\end{figure*}

\subsection{Visualization of the TopoTER Learning}

We visualize the node features $\{\H,\tH\}$ of the original graph $\A$ and transformed graph $\tA$ to observe whether the learned feature representations meet the definition of the proposed topology transformation equivariance in Eq.~(\ref{eq:ter}).
If our proposed model learns topology transformation equivariant representations, then the feature representation $\widehat{\dH}=\tD^{-\frac{1}{2}}\widehat{\dA}\tD^{-\frac{1}{2}}\X\W$ based on the estimated $\widehat{\dA}$ by our model satisfies $\H+\widehat{\dH}\approx\tH$ according to Eq.~(\ref{eq:after_transform}).

In the experiment, we adopt the real-scenario temporal graph dataset Sparrow Social \cite{rossi2015network} to evaluate our model.
Since the dataset only contains graphs at two timestamps, we treat the first graph as the original graph, denoted as $\A=\A_{t-1}$, and the second graph as the transformed graph over the original graph, denoted as $\tA=\A_{t}$.
We use a one-hot encoding of node ID as the initial node features.
We then feed the original and transformed graphs and their corresponding node features into the proposed model to train our TopoTER for the estimation of the topology transformation $\widehat{\dA}$.
Having trained our model, we feed both the original and transformed graphs with the same node features to the model to acquire the feature representations $\H=\H_{t-1}$ and $\tH=\H_{t}$, respectively.
We further feed the estimated topology transformation $\widehat{\dA}$ and the node features $\X$ to our model to infer the changed node features $\widehat{\dH}$.

Fig.~\ref{fig:vis_ter} presents the heatmaps of the leaned node features.
On the one hand, Fig.~\ref{subfig:origin_feats} and Fig.~\ref{subfig:transformed_feats} demonstrate that the difference between the node features $\H_{t-1}$ of the original graph and $\H_{t}$ of the transformed graph is dominant due to the transformation. 
On the other hand, Fig.~\ref{subfig:transformed_feats} and Fig.~\ref{subfig:origin_estimated_delta_feats} show that the difference between the node features $\H_{t}$ and $\H_{t-1}+\widehat{\dH}$ is trivial, which validates $\H_{t-1}+\widehat{\dH}\approx\H_{t}$ discussed above.
This verifies that our model learns topology transformation equivariant representations.



\section{Conclusion}
\label{sec:conclusion}

We propose a self-supervised paradigm of Topology Transformation Equivariant Representation for graph representation learning. 
By maximizing the mutual information between topology transformations and feature representations before and after transformations, the proposed method enforces the encoder to learn intrinsic graph feature representations that contain sufficient information about structures under applied topology transformations. 
We apply our model to node classification, graph classification and link prediction tasks, and results demonstrate that the proposed method outperforms state-of-the-art unsupervised approaches and reaches the performance of supervised methods at times.
We believe this model will have impact on applications such as social analysis, molecule property prediction, as well as applications in 3D computer vision.


\end{document}